\documentclass{imsart}

\RequirePackage[OT1]{fontenc}
\RequirePackage{amsthm,amsmath}
\RequirePackage{natbib}
\RequirePackage[colorlinks,citecolor=blue,urlcolor=blue]{hyperref}
\usepackage{graphicx}

\startlocaldefs
\numberwithin{equation}{section}
\theoremstyle{plain}

\endlocaldefs

\usepackage{dsfont,amssymb} 
\usepackage{color, graphicx}
\usepackage{enumerate}
\usepackage{booktabs}
\usepackage{natbib}
\usepackage{algorithm}
\usepackage{algorithmic}

\newcommand{\Pm}{\mathbb{P}}

\newcommand{\R}{\mathbb{R}^{d}}

\usepackage{amsmath}               
  {
      \newtheorem{assumption}{Assumption}
      \newtheorem{theorem}{Theorem}[section]
\newtheorem{lemma}[theorem]{Lemma}

\newtheorem{prop}[theorem]{Proposition}
\newtheorem{corollary}[theorem]{Corollary}
\newtheorem{property}{Property}

\newtheorem{definition}[theorem]{Definition}

\newtheorem{remark}[theorem]{Remark}
  }
\DeclareMathOperator*{\argminA}{argmin}

\begin{document}

\begin{frontmatter}
\title{High dimensional change-point detection: a complete graph approach}
\runtitle{High dimensional CPD: a complete graph approach}

\begin{aug}
\author{\fnms{Yang-Wen} \snm{Sun}\ead[label=e1]{yangwen.sun@hu-berlin.de}} 
\and
\author{\fnms{Katerina} \snm{Papagiannouli}\ead[label=e2]{katerina.papagiannouli@mis.mpg.de}} 

\address{Humboldt University Berlin and Max Planck Institute for Mathematics in the Sciences\\
\printead{e1,e2}}

\author{\fnms{Vladimir} \snm{Spokoiny}
\ead[label=e3]{spokoiny@wias-berlin.de}}

\address{Weierstrass Institute, Humboldt University Berlin, HSE and IITP RAS,\\
\printead{e3}}

\runauthor{Sun et al.}

\end{aug}

\begin{abstract}
 The aim of online change-point detection is for a accurate, timely discovery of structural breaks. As data dimension outgrows the number of data in observation, online detection becomes challenging. Existing methods typically test only the change of mean, which omit the practical aspect of change of variance. We propose a complete graph-based, change-point detection algorithm to detect change of mean and variance from low to high-dimensional online data with a variable scanning window. 
  Inspired by complete graph structure, we introduce graph-spanning ratios to map high-dimensional data into metrics, and then test statistically if a change of mean or change of variance occurs. Theoretical study shows that our approach has the desirable pivotal property and is powerful with prescribed error probabilities. We demonstrate that this framework outperforms other methods in terms of detection power. Our approach has high detection power with small and multiple scanning window, which allows timely detection of change-point in the online setting. Finally, we applied the method to financial data to detect change-points in S\&P 500 stocks. 
\end{abstract}


\tableofcontents
\end{frontmatter}

\section{Introduction}\label{Sec:Intro}

Change-point detection (CPD) has been widely applied in various fields such as finance (\cite{spokoiny2009multiscale}), biology (\cite{chen2011parametric}), and the Internet of Things (IoT) (\cite{aminikhanghahi2017survey}). Nowadays, as sensing and communication technologies evolves, high-dimensional data are generated seamlessly. Hence, high dimensionality, online (timely), and algorithm robustness constitute major challenges to modern change-point detection problem, and are reshaping change-point detection methodology. Motivating high dimensional change-point detection problems are, for example (a) financial structure change indicator, which can continuously monitor major market movements (\cite{grundy2020high}); (b) change of human microbiome structure and its association with life event such as pregnancy, diseases (\cite{kuleshov2016synthetic}). \par

Statistically, a change-point can be characterized as a point in sequential observations $Y_i, i=1, 2, \dots$, $Y_i \in \mathbb{R}^d$ where the probability distribution prior- and after- the sequential data are different, that is $\exists \tau > 0,  Y_i \sim \mathcal{F}_0$, for $ i< \tau$, otherwise $Y_i \sim \mathcal{F}_1$. Traditional parametric approaches have limitation for the high-dimensional data as the number of parameters to be estimated surpass the number of observations available, for example Hoteling's $T^2$ test (\cite{baringhaus2017hotelling}), and generalized likelihood ratio test (\cite{james1992asymptotic}). The assumptions needed for the distribution of each individual dimension are also difficult as the underlying distributions are normally highly context specific (\cite{siegmund2011detecting}). On the other hand, for the nonparametric approaches such as kernel-based method (\cite{harchaoui2009kernel}), the increasing dimension makes the selection of kernel function and the bandwidth a optimization process. 

To resolve the complexity of CPD problem due to dimensionality, a common approach is to project the multi-dimensional data into a metric, and then apply univariate CPD method to detect change-point. For example, \cite{wang2018high} study the optimal projection of CUSUM statistics to maximized a change in mean. In particular, the graph-based CPD method, first proposed by \cite{friedman1979multivariate}, is a two-sample test based on minimum-spanning tree (MST) representing the similarity between observations. Also \cite{rosenbaum2005exact} propose another test based on the minimum-distance pairing (MDP) using the rank of the distance within the pairs, which is thus restricted to MDP graph. Recently, \cite{chen2015graph} utilizes MST and MDP graph representations onto the data, and construct a test statistic based on counting the number of edges connecting data points before and after the potential change-point. It demonstrates better detection power in high-dimensional data compared to parametric methods. However, its detection power is comparably not sensible to the variance change and is designed for offline retrospectively detection in a closed dataset.

The key contribution of our theoretically-sound methodology is that it can detect change in mean and change in variance for online high-dimensional data in a timely accurate manner, while current methods typically are limited to the change of mean. Inspired by the complete graph structure, we devise graph-spanning ratios to map the dimensional data into metrics that have distributions corresponding to the mean and variance change of the original data. The detection of variance change can be applied to many practical problems where the volatility is an important factor. For example, the volatility presents in financial market. 

Through theoretical study, we show that the proposed graph-spanning ratios have pivotal property. It allows the approximation of test thresholds without training data and thus can effectively reduce the computation time during online detection. We show that the lower bound of minimax separation rate to testing over alternative hypothesis, is in the order of $\sqrt{nd}$, which is consistent with the rate found in \cite{enikeeva2019high}, and \cite{liu2021minimax}. We bring the spanning-ratio CPD online to perform detection in real-time setting.  Multiple scanning windows capture  the incoming data for a timely detection. Our proposed graph spanning-ratio framework allows us to detect changes online while maintains accuracy with small scanning window. The outline of this paper is as follows, Section \ref{Sec:Online} entails the completed graph based spanning-ratio algorithms for online change-point detection, and Section \ref{Sec:Theo} provides theoretical base for the algorithm, and in Section \ref{Sec:Simu} we give empirical validation of these results.

\section{Online change-point detection based on complete graph}\label{Sec:Online} 
In this section, we introduce the test statistics for the change-point detection taking into account the similarity properties from a complete graph. Then we define the $\alpha$-quantiles for the test statistics and provide the algorithms for the estimation of the critical values.\par
We observe online data: $\{Y_i\}_{i=1, 2, \ldots}$, where $Y_i \in \mathds{R}^d$. The change-point problem can be formulated as hypothesis testing, that is, to test the null hypothesis
\begin{equation}\label{H0}
H_{0}:  \quad Y_i \sim \mathcal{F}_0, \quad i=1, 2, \dots. 
\end{equation}
against the single change-point alternative
\begin{equation}  \label{H1}
	H_{1}: \quad \exists ~ \tau > 1 ,\ Y_i \sim \left\{ \begin{array}{ll} \mathcal{F}_0, & i< \tau \\ \mathcal{F}_1, & i \ge \tau, \end{array} \right.
\end{equation}
where $\mathcal{F}_0$ and $\mathcal{F}_1$ are two probability measures that differ on a set of non-zero measure. $\tau$ refers to the change-point. \par

\subsection{Notation }\label{subsec:graph_simi}
Let us now introduce some notation regarding the distributions we use. We write $\mathcal{N}(\mu, \sigma^2)$ the Gaussian distribution with mean $\mu$, and variance $\sigma^2$; $\chi^2_{df}$ as the chi-squared distribution with $df$ degrees of freedom; $F_{df_1, df_2}$ as the Fisher distribution with $df_1$, and $df_2$ degree of freedom.
We consider an undirected graph $G=(V,E)$, in which vertices $V=[n]$ represent a block of $n$ consecutive observations $\{1,\ldots,n\}$ from the sequential data. Edges set $E$ indicates the connectivity of two nodes. We define $W_{ij} $ as the square Euclidean distance between the nodes, that is $W_{ij} = \| Y_i - Y_j\|^2 $.  

\begin{definition}
The graph spanning distance of a complete graph $G$ with nodes $\{1,\cdots,n\}$  is defined as
\[ \| W_{G} \|^2 = \sum_{\{ij\} \in E} W_{ij} = \sum_{i =1}^{n}   \sum_{j =1}^{n} \|Y_i -Y_j\|^2, \]
\end{definition}

where $\| W_G \|^2 $ is the sum of squared distance between nodes in a complete graph $G$.

\begin{definition}[Graph spanning ratio (GSR)]
The graph spanning ratios are the functions $R_{\mu},R_{\sigma^-}, R_{\sigma^+}$ which map the high-dimensional observations $\mathbb{Y}_n = \{Y_1,\ldots,Y_n\} \in \mathds{R}^{d \times n}$ to the metric space, i.e., $R_{\mu},R_{\sigma^-}, R_{\sigma^+}: \mathbb{R}^{n\times d} \to \mathbb{R}$.   
\end{definition}
Figure \ref{Fig:GSR} illustrates the above definition.

\begin{figure}[h]
{%
\setlength{\fboxsep}{0pt}%
\setlength{\fboxrule}{0pt}%
\fbox{\includegraphics[width=.75\textwidth]{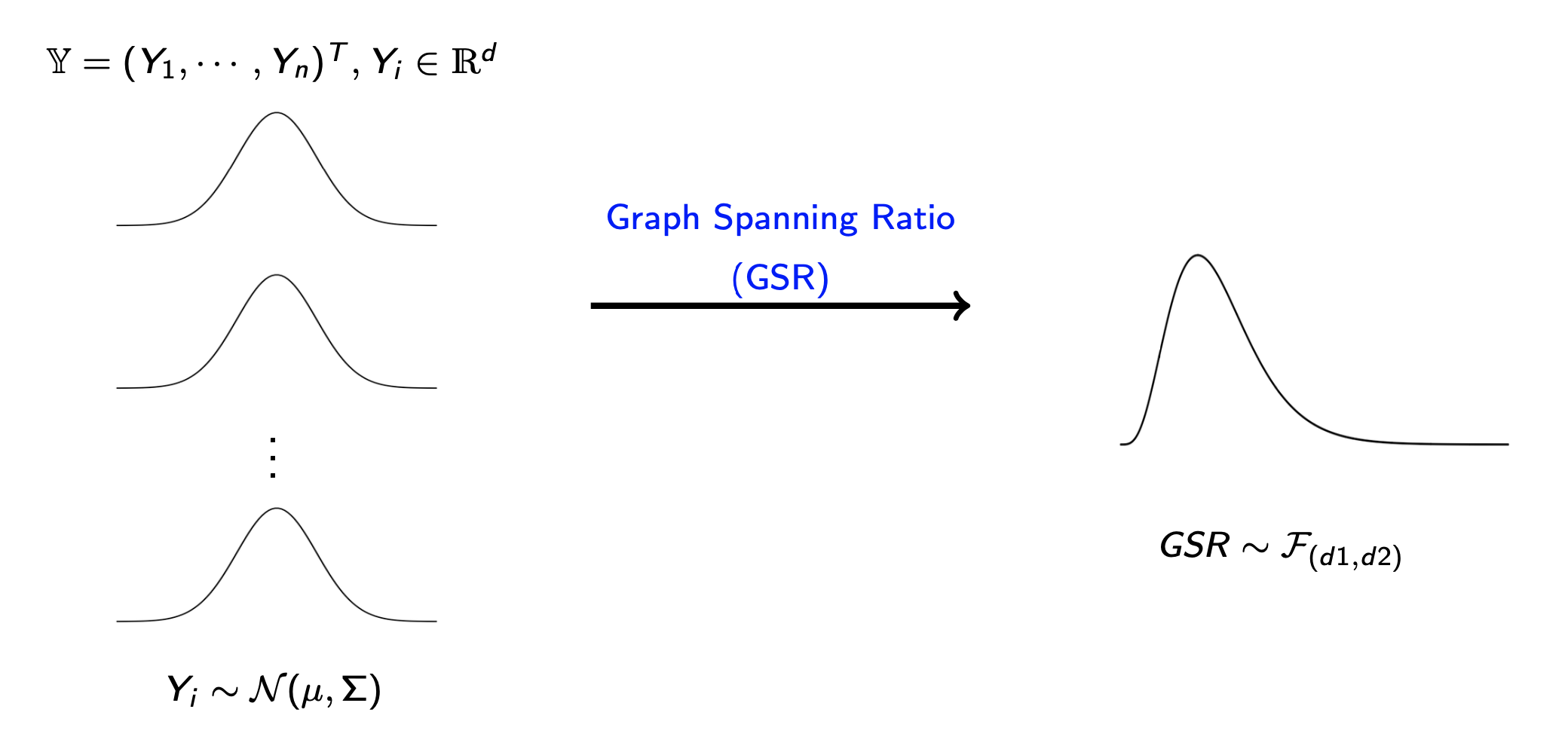}\setlength{\belowcaptionskip}{-5pt}
 }%
  \caption{Graph spanning ratio (GSR) maps the high-dimensional data into a metric.}\label{Fig:GSR}
}\end{figure}
\subsection{Graph spanning ratio}
It becomes challenging to compare the distribution $\mathcal{F}_0$ to $\mathcal{F}_1$ as the dimension $d$ increases. We devise the test statistic based on a graph-spanning ratio derived from the graph structure from the sample space of observations $\{Y_i\}$. Given a window size $n$, where $ n \in \mathbb{N}$, we set up an online scanning window  $\{Y_i:~i=t-n,\ldots,t+n-1\}$ for any candidate value $ t > n $ of the change-point $ \tau $.

We derive the test statistic based on the graph-spanning ratios of the data. In each scanning window, we divide the observations into two equally large groups: observations which come before $t$ and observations which come after $t$, i.e. the potential change-point. We construct three graphs, based on segments from the scanning window. Let $G_{2n,t}$ be the complete graph construct on nodes from data $\{ Y_{t-n},\ldots,Y_{t+n-1} \}$,  $G^{l}_{n,t}$ a complete graph consists of nodes before time stamp $t$,  with data $\{ Y_{t-n},\ldots,Y_{t-1} \}$. Respectively, $G^{r}_{n,t}$ is the complete graph consists of nodes after time $t$ within the window, that is data $\{ Y_{t},..., Y_{t+n-1} \}$. 


The graph spanning distance of complete graph $G_{2n,t}$ is\\ $ \|W_{{G}_{2n,t}}\|^2 = \sum_{i=t-n}^{t+n-1} \sum_{j=t-n}^{i-1}\|Y_i -Y_j\|^2 $, similarly for complete graph $G^{l}_{n,t}$, $G^{r}_{n,t}$. Then, we introduce the GSR for the graphical mean:
\begin{equation} \label{eq:T_mu}
R_{\mu,n}(t)= \frac{\| W_{{G}_{2n,t}}\|^2}{\| W_{G^{l}_{n,t}} \|^2+\| W_{G^{r}_{n,t}}\|^2}
\end{equation}

and the graph-spanning ratio for the graphical variance:
\begin{equation}\label{eq:T_var}
  R_{\sigma-,n}(t)= \frac{\| W_{G^{l}_{n,t}}\|^2}{\| W_{G^{r}_{n,t}}\|^2}, \hspace{4pt}  
  R_{\sigma+,n}(t)= \frac{\| W_{G^{r}_{n,t}}\|^2}{\| W_{G^{l}_{n,t}} \|^2}     
\end{equation}  

The complete graphs $G_{2n}, G^{l}_{n}, G^{r}_{n}$ are constructed with the squared Euclidean distance between nodes. Note that the graph-spanning ratio of the graphical mean $R_{\mu,n}(t)$ is devised in such a way that it increases when a change of mean occurs. Similarly for $R_{\sigma+,n}(t)$ and $R_{\sigma+,n}(t)$. Figure \ref{fig:CG} gives an example of the complete-graph structure when (a) a change of mean and (b) a change of variance occurs.

\begin{figure} 
  \textbf{(a)} \hspace{102pt} \textbf{(b)}\\
  \includegraphics[width=.75\textwidth]{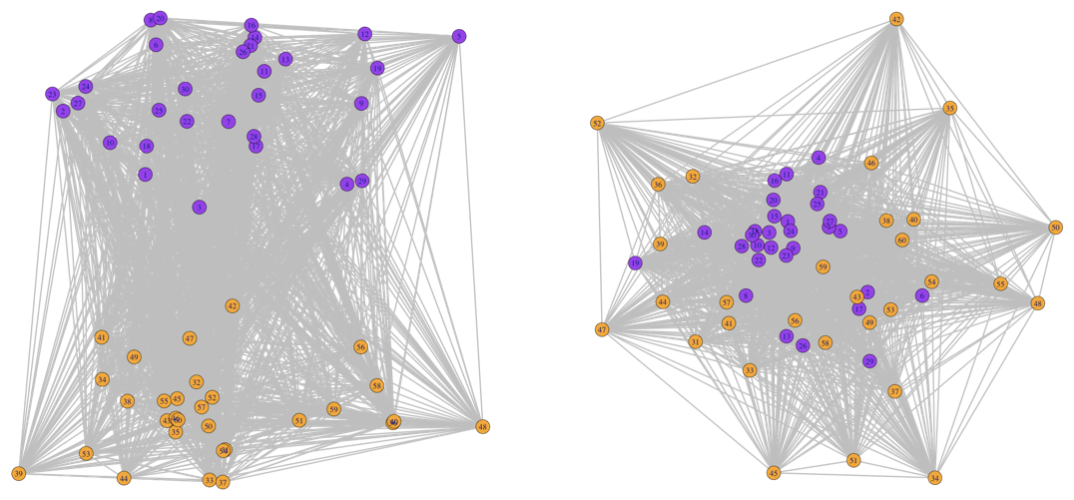}
  \caption{Graph representation of a two-dimensional sequential data. Complete graphs are constructed from 60 i.i.d. normal distributed observations with first 30 observations (in orange) from standard normal, the second 30 observations (in purple) with (a) change in mean, (b) change in variance. }
 \label{fig:CG}
\end{figure}

Here, we define $\rho_{\mu,n}(\alpha)$ as the $\alpha$-quantile of the GSR metric $R_{\mu,n}$ as
  \[ \rho_{\mu,n}(\alpha)= \argminA_{\rho} \{\mathds{P}\big( R_{\mu,n}(t)  \geq \rho \big) \leq \alpha \} \]
  
Similarly, we define the $\alpha$-quantile for $R_{\sigma-,n}(\alpha)$ and $R_{\sigma+,n}(\alpha)$ as $\rho_{\sigma-,n}(\alpha)$, and $\rho_{\sigma+,n}(\alpha)$.
 
                
The  $\alpha$-quantile, $\rho_{\mu,n}(\alpha)$, serves as the critical point to test the change of graphical mean. Then we define the test-statistics of the mean
 \[ T_{\mu,n} (t) =  R_{\mu,n}(t) - \rho_{\mu,n}(\alpha) \]
 and similarly for the test-statistics of the variance $T_{\sigma-,n}(t)$, and $ T_{\sigma+,n}(t)$.
The choice of $\alpha$-level can vary for test mean and test of variance, as described in Section \ref{section:P1}.

The critical values $ \rho_{\mu,n}(\alpha), \rho_{\sigma-,n}(\alpha),\rho_{\sigma+,n}(\alpha)$ under the null distribution $H_0$ can be determined by the tail behavior of GSR ratios. Later, we show that the distribution of the GSR ratios $ R_{\mu,n}, R_{\sigma-,n}, R_{\sigma+,n}$ can be expressed analytically when the data are i.i.d and Gaussian distributed. \par 

The thresholds of the testing exhibit a pivotal property, i.e., the threshold values are not dependent on the unknown Gaussian parameters such as mean and variance, due to the geometry of a complete graph and the ratio structure of the test-statistics. Thus, we can predetermine them before testing. This pivotal property is the advantage of the proposed method over data-driven approaches like bootstrap sampling. We do not need to train the data in advance.  

Therefore, two-sample hypothesis tests with test statistics $T_{\mu,n}(t)$, $T_{\sigma-,n}(t)$, and $T_{\sigma+,n}(t)$ can be used to detect change in graphical mean and variance. Given a significant level $\alpha$, we reject $H_0$, if $t>n$, any of the pooled test statistics is greater than 0, i.e.,
                  \[ \max(\mathbb{T}_{\mu}(t),\mathbb{T}_{\sigma+}(t),\mathbb{T}_{\sigma-}(t)) \geq 0,\]
where the pooled statistics is defined as
                  \[  \mathbb{T}_{\mu}(t):=\sup_{n \in \mathfrak{N}} \{ R_{\mu,n}(t) -  \rho_{\mu}(\alpha) \},  \]
similarly, we define $ \mathbb{T}_{\sigma+}(t)$, and $ \mathbb{T}_{\sigma-}(t)$.

\subsection{Estimation of the thresholds for the online detection}
In the online setting, we need to take into account the small sample dependency structure, see for example \cite{kirch2008bootstrapping}. More precisely, the consecutive scanning statistics $T_{\mu, n}(t), T_{\mu, n}(t+1)$ are correlated due to the fact that we receive the data sequentially. To circumvent this problem, we define a zone for the scanning window where we taking the maximum over this zone. More precisely, we generate a sequence of $N$ i.i.d. standard Gaussian random variables $Y_1,\ldots,Y_N$, with $N \geq 2n$. Then, we define a zone $A_n=\{n+1,\ldots, N-n+1\}$, where $n$ is the size of the scanning window. Given a fixed window size $n$, we calculate the similarity metrics ${R_{\mu,n}}(\mathbb{Y}_t)$, ${R_{\sigma+,n}}(\mathbb{Y}_t)$, and ${R_{\sigma-,n}}(\mathbb{Y}_t)$ for $t \in A_n $. Finally, we define  
         \[ R_{\mu,n}^{max}:= \max_{t \in A_n} R_{\mu,n}(t), \]
similar definition applied to $R_{\sigma+,n}^{max}$, and $_{\sigma-,n}^{max}$.

We repeat the procedure multiple times and compute the $\alpha$-quantile among the similarity metrics.
         \[ \rho_{\mu,n}(\alpha) :=\inf\{ x: \mathds{P}^{M} 	\big( R_{\mu,n}  \geq x \big) \leq \alpha \},\]
where $ \mathds{P}^M$ denote the probability measure under Monte Carlo simulation, since we approximate the online detection threshold $\rho_{\alpha,\mu,n}, \rho_{\alpha,\sigma,n}$ using Monte Carlo simulations. As the number of simulation increases, by the central limit theorem, the estimated value is approaching the true threshold value (\cite{dong2020tutorial}). Similarly, we estimate $\rho_{\sigma+,n}(\alpha) $ and $\rho_{\sigma-,n}(\alpha) $. Detailed online detecting procedures are specified in Algorithm \ref{algo.onlineRao} and Algorithm \ref{algo.onlineDetection}. The complexity of computing the threshold is $\mathcal{O}(n^3d)$. The threshold can be determined before online detection because of the pivotal property of the test-statistics. The complexity of the online detection algorithm is $\mathcal{O}(n^2d)$. The threshold depends only on the window length $n$ and data dimension $d$, as we will show in the theoretical properties section.
\begin{algorithm}[tb]
   \caption{ Online scanning threshold}\label{algo.onlineRao}
\begin{algorithmic}
   \STATE {\bfseries Input:} $N$ i.i.d standard Gaussian data $Y$ of dimension $d$,  $N > n$, significant level $\alpha$
   \STATE Initialize arrays $R_{\mu,n}, R_{\sigma-,n},R_{\sigma+,n}$.
   \FOR{$i=1$ {\bfseries to} $K$}
   		\FOR{$j=n+1$ {\bfseries to} $N-n+1$}
   		\STATE Calculate $R_{\mu,n}(t),  R_{\sigma+,n}, R_{\sigma-,n}$
  		 \ENDFOR
   \STATE Calculate $R_{\mu,n}^{max}[i],R_{\sigma+,n}^{max}[i],R_{\sigma-,n}^{max}[i]$
   \ENDFOR
   \STATE Calculate $ \rho_{\mu,n}(\alpha)$ as the $(1-\alpha)$ quantile of $R_{\mu,n}^{max}$
   \STATE similarly for $ \rho_{\sigma+,n}(\alpha)$ , $ \rho_{\sigma-,n}(\alpha)$
   \STATE \textbf{return} $ \rho_{\mu,n}(\alpha)$, $ \rho_{\sigma+,n}(\alpha)$, and $ \rho_{\sigma-,n}(\alpha)$.
\end{algorithmic}
\end{algorithm}

\begin{algorithm}[tb]
   \caption{ Online change-point detection}\label{algo.onlineDetection}
\begin{algorithmic}
   \STATE {\bfseries Input:} Online data $\{ Y_{t-2n+1},...Y_{t-n}, Y_{t-n+1}, ...,Y_{t}\}$ 
   \STATE Initialize $I_{\mu}=I_{\sigma+}=I_{\sigma-}=0$    
   \REPEAT
   \FOR{$n \in \mathfrak{N}$}
   \STATE Calculate $R_{\mu,n}, R_{\sigma+,n}$, and $R_{\sigma-,n}$
   \IF{$R_{\mu,n}> \rho_{\mu,n}(\alpha)$}
   \STATE $I_{\mu}=1$ \textbf{return} Mean change at $t-n+1$
   \ENDIF
      \IF{$R_{\sigma+,n} >  \rho_{\sigma+,n}(\alpha)$}
   \STATE $I_{\sigma+}=1$ \textbf{return} Variance increased at  $t-n+1$
   \ENDIF
      \IF{$R_{\sigma-,n} >  \rho_{\sigma-,n}(\alpha)$}
   \STATE $I_{\sigma-}=1$ \textbf{return} Variance decreased at $t-n+1$
   \ENDIF
   \ENDFOR
   \UNTIL{$ I_{\mu}+I_{\sigma+}+I_{\sigma-}>0$}
\end{algorithmic}
\end{algorithm}

\section{Theoretical properties}\label{Sec:Theo}
In this section, we derive theoretical results for the quality of the GSR test statistic and the detection power of the GSR, in case the observations before the change-point follow the Gaussian distribution. Consider the online scanning scheme we described in the previous section. Let $ 2n $ be the length of the scanning window length with a finite collection of window lengths: $ n \in \mathfrak{N}$.
\subsection{Theoretical properties of the GSR metrics}
Now, we introduce some assumptions regarding the distribution of the observations. Generally, we assume the nodes of graph $ G_{n} $, namely $ Y_i$ to be i.i.d. random variables normally distributed. However, we suppose two different cases for the form of the covariance matrix. The nodes are paired with the Euclidean distance in $ \R $.
	\begin{assumption}[i.i.d.]\label{iid}
		The nodes of  graph,  namely $ Y_i$, are independent identically distributed (i.i.d) random variables.
	\end{assumption}
	
		\begin{assumption}[Constant variance]\label{gaussion}
		The nodes of the graph, $ Y_i$, are normally distributed with mean $\mu$, and variance $\sigma^{2} I $, that is  $ Y_i \sim\mathcal{N}\left(\mu, \sigma^{2} I \right)$ and $ Y_i \in \R$.
	\end{assumption}
	
	\begin{assumption}[Uncorrelated covariance matrix]\label{homovar}
		The nodes of the graph, $Y_i \sim\mathcal{N}\left(\mu, \Sigma \right)$, $\Sigma \in \mathbb{R}^{d \times d}$. The covariance matrix satisfies $\Sigma_{j,j} =\sigma_j, j =  1,\ldots,d$. $\Sigma_{i,j}=0 $, for $i \neq j$.
	\end{assumption}

Next, we show some theoretical properties concerning the distribution of the GSR test statistics and the spanning distance of the graph. Without loss of generality, we leave out the time stamp $t$ in this section for concise expression.

At this point let us introduce a definition concerning the spanning distance of the remaining terms between $G_{2n}$ and $G^{l}_{n}$, $G^{r}_{n}$. The role of this definition is twofold. First it will enable us to prove independence between the components of the graph spanning ratios GSR. Furthermore, this quantity is useful in order to determine how far the mean separates the data before and after the change point, as we shall see in Theorem \ref{Theo:Delta_mu}. 
\begin{definition}
We define a residual spanning distance: $\| W_{rem,n}\|^2 =  \| W_{G_{2n}}\|^2 -  2 (\|W_{G^{l}_{n}}\|^2 - \|W_{G^{r}_{n}}\|^2 ) $ which is the total spanning distance of $G_{2n}$, excluding the spanning distance of $G_n^l$ and $G_n^r$.
\end{definition}

\begin{lemma} \label{sp_idep}
The two metrics  
		$ \| W_{rem,n}\|^2 $ and $(\|W_{G^{l}_{n}}\|^2 - \|W_{G^{r}_{n}}\|^2)$ are linearly independent.
\end{lemma}
	
Next, we show the probability distribution of the spanning distance ratios follows a Fisher distribution.

\begin{prop}\label{lemma:testmu}
Suppose that Assumption 1 and 2 hold. Then, the similarity metric for the graphical mean follows Fisher distribution. Precisely, 
\begin{align*}
R_{\mu,n} - 2 
& \sim \frac{1}{(n-1)} F_{d, 2(n-1)d},
\end{align*}

where $n$ and $d$ are the window size and dimension of data respectively.
	\end{prop}
	
	\begin{prop}\label{proposition:fisher1}
	The similarity metric for the graphical variance follows the Fisher distribution. Precisely, 
	$R_{\sigma+, n}$ and $	R_{\sigma-, n} \sim F_{(n-1)d, (n-1)d}$  

	where $\| W_{G^{l}_{n}}\|^{2}, \| W_{G^{r}_{n}}\|^{2}$ are the distances spanned by a complete graph before and after a potential change point respectively.
\end{prop}	
\begin{remark}
The degrees of freedom of the fisher distribution are the same for both ratios. When the number of vertices are the same in graph $G^{l}_{n}$ and $G^{r}_{n}$, the degree of freedom for the numerator and denominator are the same as well.
\end{remark}
\begin{property}[Pivotal property]
Under Assumption \ref{iid}  and \ref{gaussion}, the graph-spanning test-statistics $T_{\mu,n}, T_{\sigma-,n}$, and $T_{\sigma+,n}$ are independent of the unknown Gaussian parameters $(\mu, \sigma)$. 
\end {property}
Let us mention here that, for the data with constant variance, the above property allows us to determine the test-thresholds without training data. Next, we present a result similar to \ref{proposition:fisher1}. This time we use Gaussian data with uncorrelated covariance matrix.

\begin{prop}\label{lemma:testmuIH}
Let Assumption \ref{iid}, \ref{gaussion}, and \ref{homovar}  hold. Then,
\begin{align*}
R_{\mu,n}- 2 
& \overset{\text{approx}}{\sim}\frac{1}{(n-1)} F_{\upsilon , 2(n-1)\upsilon}, 
\end{align*}
\[R_{\sigma+, n} \overset{\text{approx}}{\sim} F_{(n-1)\upsilon, (n-1)\upsilon}\]
\[R_{\sigma-, n} \overset{\text{approx}}{\sim} F_{(n-1)\upsilon, (n-1)\upsilon}\]
where $ \upsilon = \frac{(\sum_{i=1}^{d} \sigma_i^2)^2}{\sum_{i=1}^{d} \sigma_i^4}$.
\end{prop}
The proof of above proposition can be shown by applying theorems from \cite{kim2006analytic} and Welch –Satterthwaite equation. The degree of freedom component,$\upsilon$, $ 1 \leq \upsilon \leq d$, depends on the the variances $\sigma_i$, The closer $\sigma_i$'s to each other, the closer $\upsilon$ is to data dimension $d$.

We sum up the previous discussion to two important properties for the GSR test ratios regarding their distribution in a non-asymptotic and asymptotic regime.
 \begin{property}
 GSR function $R_{\mu}, R_{\sigma^-}$, and $R_{\sigma^+}$ maps the Gaussian distributed random variable $Y_i$ into Fisher distributed random variables.

\end{property}
\begin{property}[asymptotic of GSR]
For large $n$, the distribution of GSR for the mean and variance are approximately Gaussian.
\end{property}

\subsection{Quality of the GSR test}
The quality of a test $\phi$ is typically measured by the error probabilities of type I (false positive) and type II (false negative). For the behaviour of the test under the null hypothesis $H_0$,  the type-I error probability is $ \alpha = \textbf{P}_0 (\phi=1)$

The $\alpha$ value is defined as the level of the test, which is given in the beginning of the test, by which the probability of rejecting $H_0$ when $H_0$ is true is aimed to be controlled. \
If a change-point exists, i.e.  $Y_i \sim \mathcal{F}_1, i \geq t$, then the error probability of the type II is defined as $ \beta= \textbf{P}_1(\phi=0)$
The value of $1-\beta$ is called the power of the test $\phi$ at $\mathcal{F}_1$.

In this section, the $\alpha$-level and $(1- \beta)$ power are theoretically verified to ensure the quality of the proposed GSR test. 

\subsubsection{Level of the test for multiple windows} \label{section:P1}
For multiple window and online testing, let us fix some $\alpha \in (0,1)$.  For any $ t > n $, we denote the pooled test-statistic $\mathds{T}_{\mu}$ as 
\begin{equation} \label{eq:poolT}
 \mathds{T}_{\mu}=\sup_{n \in \mathfrak{N}} \Big\{  T_{\mu,n} \Big\} =\sup_{n \in \mathfrak{N}} \Big\{  R_{\mu,n} - \rho_{\mu,n}(\alpha_{\mu,n}), \Big\}
\end{equation}
where  $\rho_{\mu,n}(\alpha_{\mu,n})= \argminA_{\rho} \{\mathds{P}\big( R_{\mu,n}  \geq \rho \big) \leq \alpha_{\mu,n} \} $
 $ \{ \alpha_{\mu,n}, n \in \mathfrak{N} \}$ is a collection of numbers in $(0,1)$, such that $\forall Y_i \sim \mathcal{F}_0$, $i \in G_{2n}$, 
$ \Pm_0 ( \mathds{T}_{\mu} > 0 ) \le \alpha_{\mu}.$

We reject the null hypothesis when $ \mathds{T}_{\mu} >0$. For each window size $n$, we choose the significant level $\alpha_{\mu,n}$ according to the following procedure: 

\textbf{P1.} The sum of $ \alpha_{\mu,n}$ satisfies the equality:  $\sum\limits_{n \in \mathfrak{N}} \alpha_{\mu,n} = \alpha_{\mu}$.

\textbf{P2.} The sum of $ \alpha_{\mu,n}$ satisfies the equality $ \alpha_{\mu} + \alpha_{\sigma-} + \alpha_{\sigma+} = \alpha_{total}$.

In practice, when we adapt \textbf{P1} procedure, we choose $\alpha_{\mu,n} = {\alpha_\mu}/|\mathfrak{N}|$ for all $n \in \mathfrak{N}$, where $|\mathfrak{N}|$ is the cardinality of the windows.

This is to say that, the level of the test, $\alpha(\phi)$, i.e., false positive rate, is in theory controlled at $\alpha_{total}$ under the \textbf{P1} and \textbf{P2} procedures. For complete graph,  the GSR ratio $ R_{\mu,n}$ follows Fisher distribution, as shown in Proposition \ref{lemma:testmu}.


The pooled test statistics for all window size
\begin{align*}
 \mathds{T}_{\mu}=&\sup_{n \in \mathfrak{N}} \Bigg\{ \frac{\|W_{G_{2n}}\|^2 }{\Big( \|W_{G^{-}_{n}}\|^2 + \|W_{G^{+}_{n}}\|^2 \Big)} \\
 &- 2 - \frac{1}{(n-1)} F^{-1}_{d, 2(n-1)d} (\alpha_{\mu,n}).\Bigg\}
\end{align*}

The choice of $\alpha$ as stated in procedure \textbf{P1} allows us to compute the $\alpha$-quantile of the test. Through this procedure,  no preliminary computation such as Fisher $\alpha-$quantile,  is required.  This test procedure \textbf{P1} gives a Bonferonni test result,  that is, 
\begin{align*}
& \Pm_0( \mathds{T}_{\mu} > 0 ) \\
 & \le \sum\limits_{n \in \mathfrak{N}} \Pm_0 \Bigg( \frac{ (n-1)  \|W_{G_{rem,n}}\|^2 }{\Big( \| W_{G^{-}_{n}}\|^2 +\|W_{G^{+}_{n}}\|^2 \Big)} > F^{-1}_{d, 2(n-1)d} (\alpha_{\mu,n})  \Bigg) \\
&\le \sum\limits_{n \in \mathfrak{N}} \alpha_{\mu,n} =  \alpha_\mu.
\end{align*}

$\alpha_{\sigma+}$, and $\alpha_{\sigma-}$ can be obtained in a similar way. By procedure \textbf{P2}, the level of the test 
\[\alpha(\phi) = \Pm_0 ( max (\mathds{T}_{\alpha,\mu},\mathds{T}_{\alpha,\sigma+}, \mathds{T}_{\alpha,\sigma-})   > 0 ) \le \alpha_{total} \]
 is hence satisfied.

\subsubsection{Power of test for multiple windows}
In this section, we wish to know the detection power the test. In other words, we like to know the quality of the test to detect a change, when the change of distribution, is greater than a threshold quantity.  

In the following theorem, we will focus on the probability of detecting the change of distribution in term of mean change (separation gap between graphs). With $\Pm$-probability greater than $1-\beta$, $\beta \in (1,0)$  the test is able to detect the change, when the change of mean if greater than the threshold quantity. This gives the false negative rate $\beta$.

\begin{theorem}[Power of the test]\label{Theo:Delta_mu}
Let $ \mathds{T}_{\mu}$ be the test statistics specified in Equation (\ref{eq:poolT}), and $\beta \in (0,1)$. Then
$\Pm (  \mathds{T}_{\mu} > 0) \ge 1-\beta$, if 
 \[\sup_{n \in \mathfrak{N}} \{ \| \mu_{rem,n} \| ^2 - \Delta_{\mu}(n) \} \geq 0.\]
\begin{align*}
\Delta_{\mu}(n) = C_1 \Big( \| \mu_{l,n}\|^2+\| \mu_{r,n} \|^2 +  C_2 \sigma^2 \Big) .
\end{align*}
where $\| \mu_{rem,n} \| ^2$,   $ \|\mu_{l,n}\|^2$,  and $\| \mu_{r,n} \|^2$ are the expected residual spanning distances and the expected spanning-distance of subgraph $G^{l}_n$ and  $G^{r}_n$, respectively.   

\begin{align*}
 C_1 = &5 \frac{N_n}{D_n} F^{-1}_{N_n, D_n} (\alpha_{\mu,n}),  \\
 C_2 = 
& \Bigg( D_n + 2 \sqrt{D_n \log\Big(\frac{2}{\beta}\Big)} +4 \log\Big(\frac{2}{\beta}\Big) \Bigg)\\
& - \frac{5}{4} \Bigg(N_n  - 2 \sqrt{N_n \ log \Big(\frac{2}{\beta}\Big)} - 10 \log \Big(\frac{2}{\beta}\Big) \Bigg),
\end{align*}
$N_n = d$ and $ D_n = 2(n-1)d$.
\end{theorem}

This means that there exists a window size $n \in \mathfrak{N}$,  when the mean residual spanning distance is greater than the threshold $\Delta_{\mu}(n)$, then $ \Pm (  \mathds{T}_{\mu} > 0) \ge 1-\beta $ is satisfied. In other words,  the false negative rate (type II error) is smaller than $\beta$. Figure \ref{fig:Delta_mu} shows the $\Delta_{\mu}$ in relation to $\beta$ with fixed window length $n=30$, data dimension $d=100$, and significant level $5\%$. As the $\beta$ decreases, mean residual spanning distance $\Delta_{\mu}$ is required to be large to guarantee $1 - \beta$ true positive rate. 
\begin{figure}[h]
\centering
  \includegraphics[width=.65\textwidth]{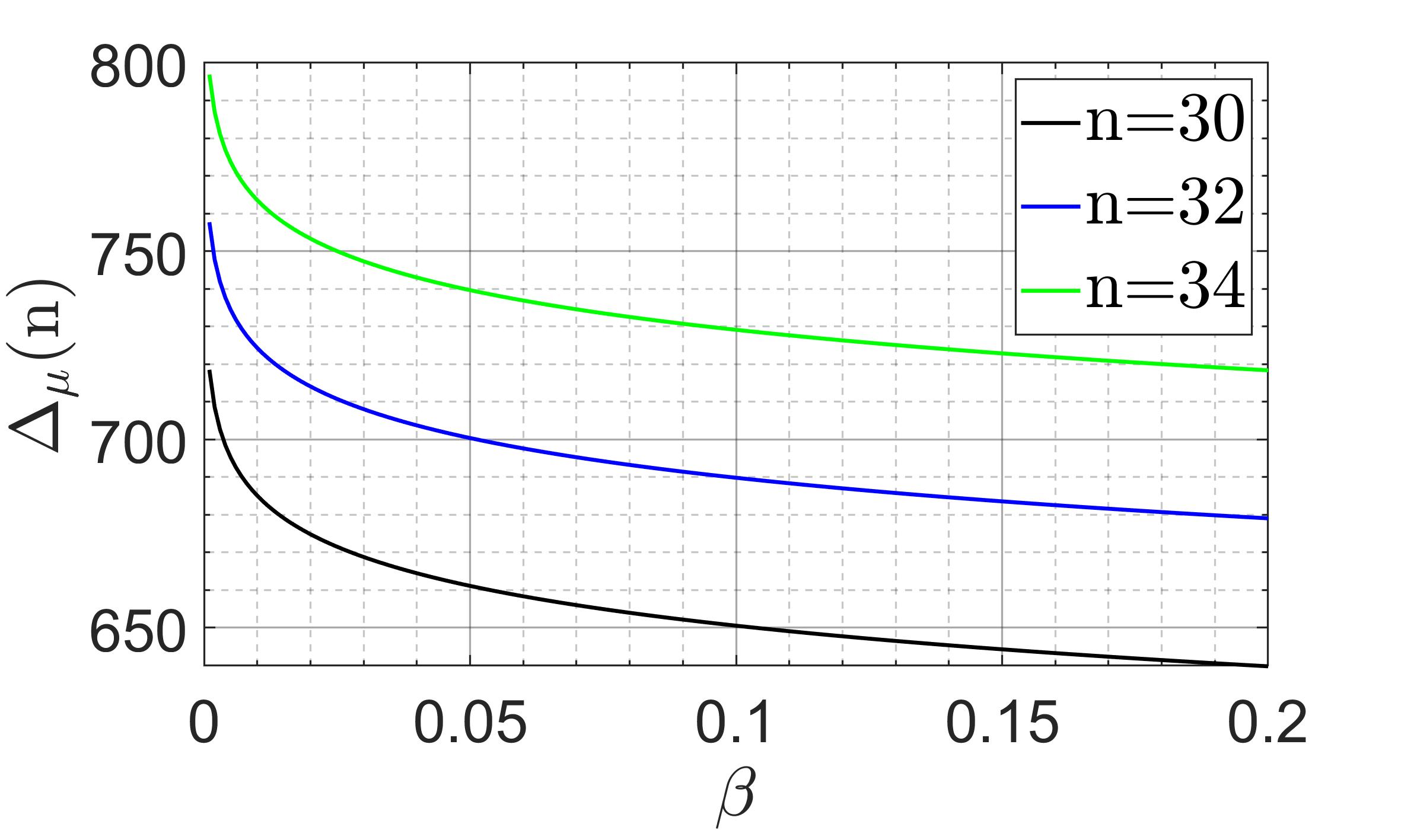}\setlength{\belowcaptionskip}{-5pt}
  \caption{Mean residual graph spanning distance $\Delta_\mu$ to ensure $1-\beta$ power in detecting mean change for window size $n=30, 32, 34$, dimension $d=10$.} \label{fig:Delta_mu}
\end{figure}

Power of the test for change of variance $\sigma-$ can be expressed in a similar way. The $(1-\beta)$-power of the test is thus can be obtained. Proofs are shown in the supplemental material.

Next section, we verify that the thresholds derived for $\|\mu_{rem,n}\|^2 $, $ \| \mu_{l,n} \| ^2$ and $ \| \mu_{r,n} \| ^2$ are greater than the minimum radius for the prescribed error rate. 

\subsubsection{Minimum radius of the mean separation}
We denote the quantity 
\[ \beta(\mathcal{F}_1 )= \inf_{\phi_{\alpha}} \sup_{ \mathcal{F}_1 } P[\phi_{\alpha}=0], \]
where $\mathcal{F}_1$ is the alternative distribution as stated in $H_1$. $\beta(\mathcal{F}_1)$ is the infimum being taken over all test $\phi_\alpha$ with values in $\{0,1\}$ satisfying $P_0[\phi_\alpha = 1] \leq \alpha$.

Let $\|\mu_{rem,n} \|^2$ belong to some subset of the Hilbert space,
$\textit{l}_2(n)=\big\{ \|\mu_{rem,n} \|^2 < \infty\big\}.$ For the problem of detecting of mean change, the minimal radius $\rho$ (i.e. lower bound of minimax separation rate) is a quantity, when $\| \mu_{rem,n} \| \geq \rho$ which the problem of testing for $i> t$, $H_0$ against the alternative,  $H_1: Y_i \sim \mathcal{F}_1$, with prescribed error probabilities is still possible (\cite{spokoiny1996adaptive}). The test $\phi_0$ is powerful if it rejects the null hypothesis for all $\{Y_i, i> t \} \sim {\mathcal{F}_1} $ outside a small ball with probability close to 1. 

We derive the minimal radius based on the result from \cite{baraud2003adaptive}, and \cite{baraud2002non}.

\begin{prop}[$(\alpha,\beta)$ minimum radius]\label{Prop:lower_bound}
Let $\beta \in (0, 1-\alpha_{\mu,n})$ and fix some window size $n \in \mathfrak{N}$
\[\theta(\alpha_{\mu,n}, \beta) = \sqrt{2 \log (1+4(1-\alpha_{\mu,n}-\beta)^2)} \]
If $\| \mu_{rem,n}\|^2 \leq \theta(\alpha_{\mu,n}, \beta) \sqrt{nd} \sigma^2, $ \\
then $\Pm(T_{\mu,n}(t) \geq 0 ) \leq 1- \beta $.
\end{prop}

Therefore, $\theta(\alpha_{\mu,n}, \beta) \sqrt{nd} \sigma^2$ is the minimum radius $\rho$. The minimum radius $\rho$ is of order $\sqrt{nd}$, which is consistent with the results from \cite{enikeeva2019high} and \cite{liu2021minimax}. The threshold derived for $\| \mu_{rem,n}\|^2$ in Theorem \ref{Theo:Delta_mu}, $\Delta_\mu(n)$ is greater than the lower-bound of the minimum radius, hence Proposition \ref{Prop:lower_bound} holds. The pooled test based on $\mathbb{T}_\mu$ has power greater than $1-\beta$ over a class of window length $\mathfrak{N}$. Thus, the test of mean change is powerful. Similarly, we can confirm that the test of change of variance is powerful. 

\section{Numerical studies}\label{Sec:Simu}
The numerical studies consist three parts: first, the comparison of detection powers with other method, and second part, the online detection power of our proposed method. Final part is the  application of the algorithm to real-world data from S\&P 500 stocks.

\subsection{Comparison of detection power - static point test}  
Our first goal is to see if our proposed method has improved detection power over other methodology, especially with the aforementioned graph-based method as stated in Section \ref{Sec:Intro},. 
To quantify the detection power of our proposed method, we consider the scenarios that the observation follow certain parametric distribution. We generate 100 samples for detection power comparison. Each sample is consist of $n$ simulated i.i.d observations, n is even. With equal probability, it follows $d$ dimensional standard normal distribution $Y_i \overset{i.i.d}{\sim} \mathbf{N}(0,I_d), i =1,\dots, n$ or from the model:
\begin{equation} 
	 Y_i \overset{i.i.d}{\sim} \left\{  \begin{array}{ll} \mathbf{N}(0,I_d), \quad i =1,\dots, n; \\ 
	                                                     \mathbf{N}(\Delta, \Sigma), \quad i =n+1,\dots, 2n. \end{array} \right.
\end{equation}\label{eq:simu_offline}
\textsc{\begin{figure}[t]
 \hspace {6pt}   cGSR \hspace{60pt}   iBGEC
  \centering 
  \includegraphics[width=.80\textwidth]{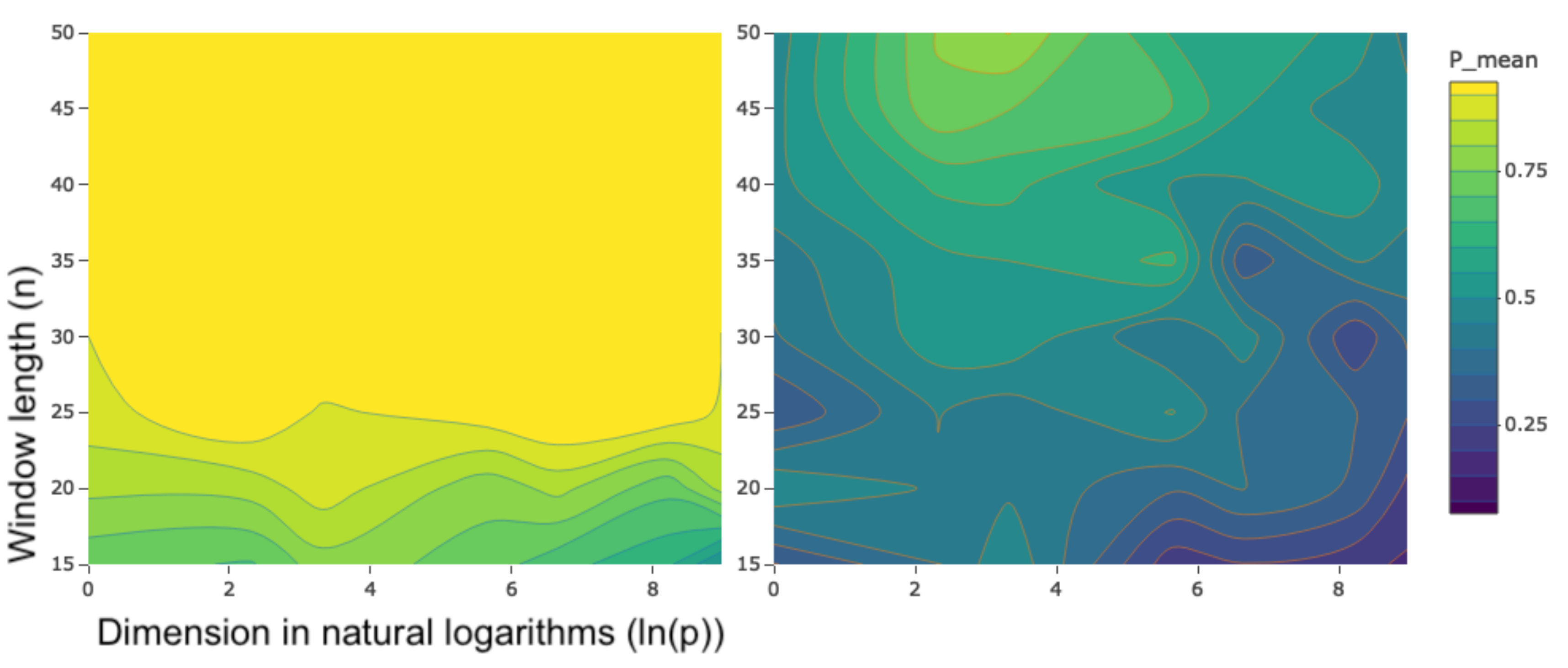}\setlength{\belowcaptionskip}{-5pt}
  \caption{Detection power P\_mean for mean change of $\Delta = 1/ \sqrt[3]{d}$, $d$ is the dimension of the data. Comparison cGSR$_{CG}$ and iBGEC with respect to dimension and window length, with significance less than 5$\%$.}
 \label{fig:Contour_Static}
\end{figure}
\begin{figure}[t]
  \hspace {6pt}   cGSR$_{CG}$ \hspace{60pt}   iBGEC
  \centering 
   \includegraphics[width=.80\textwidth]{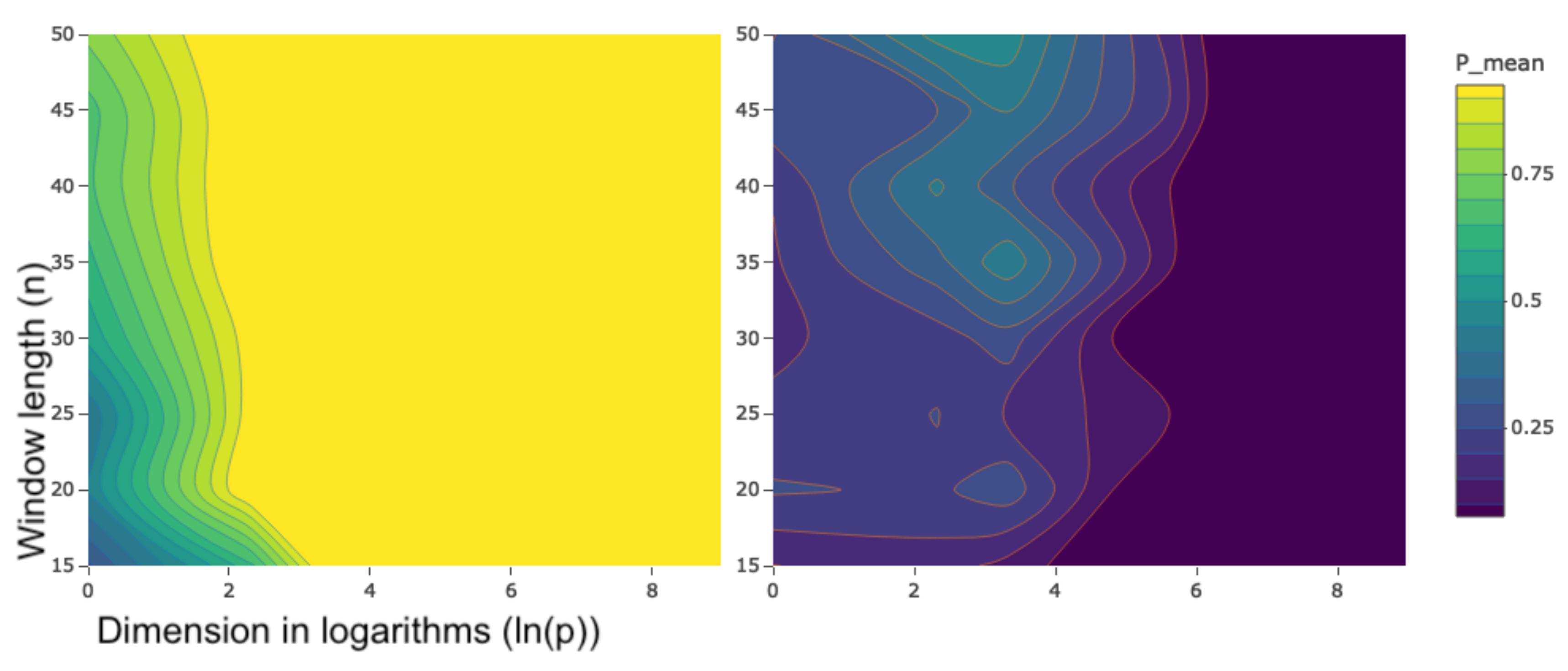} \setlength{\belowcaptionskip}{-5pt}
  \caption{Detection power P\_mean for variance change of $\Sigma = 2 I_d$. Comparison between methods cGSR and iBGEC with respect to dimension and window length, with significance less than 5$\%$.}
 \label{fig:Contour_StaticV}
\end{figure}
}
We denote our proposed complete graph-spanning-ratio methodology as cGSR to the in-between-group edge counting denoted as iBGEC algorithm proposed by Chen and Zheng in Section \ref{Sec:Intro} with various scanning window lengths.

Since iBGEC is not an online algorithm, we compare the detection power for static point test. No online-scanning applied. With this single-window setup, we can compare the results to  performance of the algorithms in terms of detecting change-point.  We consider both accuracy and sensitivity as a general way of comparing detection power (\cite{aminikhanghahi2017survey}). Detection accuracy is defined as how often the detection algorithm make right decision, that is, to identify change-point when there in reality a true change-point, and identify no change-point when there is true non-change-point. We denote $TP$ as true positive, $FN$ as false negative, and so on. Then we define accuracy$=\frac{TP+TN}{TP+TN+FP+FN}$. We denote FPR $= \frac{FP}{FP+TN}$ as the false positive rate which is rate of giving a false alarm when no change-point present.

For detection sensitivity, we concern about the success rate of identify a change-point when there indeed true change-point exist. Therefore, sensitivity $= \frac{TP}{TP+FN}$. To consider the detection power with both the accuracy and the sensitivity of the detection methods, here we define a power metric as the geometric mean of the accuracy and sensitivity, P\_mean $= \sqrt{accuracy \times sensitivity}$

In each of the 100 sample generated, it contain either with or without change-point in the middle of sample, as shown in (\ref{eq:simu_offline}). Note that for this power determination simulation, each sample is of the same length of the scanning window $n$, so in this setting, the change-point occurs at location $n$ of the data, if there is any. We compare the accuracy results, denote as cGSR, to result from in-between-group edge counting algorithm (iBGEC) by Chen and Zheng et al. We depict the detection power with respect to dimension of data $d$, and length of window $n$. Figure \ref{fig:Contour_Static} is the detection result of a mean change $\Delta=1/ \sqrt[3]{d}$. The detection power are higher for cGSR compared to iBGEC method, across all dimension and window length. In Figure \ref{fig:Contour_StaticV}, a significant improvement in detecting variance change, in particular, with small window length under high-dimensional scenarios. This make our proposed algorithm more ideal for further online detection, where a timely detection of change-point is important.

\subsection{Detection power of the online algorithm}
In this section, we examine the online detection power of the proposed online cGSR algorithms, as stated in Algorithm 1, and 2. We generate 1000 samples for testing, each sample is consist of 100 observations, with first-half of observations follow $p$ dimensional standard normal distribution, and second-half of observations, with equal probability, either remains the same distribution or has a change in its distribution.  We study the online detecting power and false alarm rate, i.e. FPR, with multiple scanning window length. Scanning window from short- to long- length $\{20, 35, 50\}$ are applied into the online detection algorithm. In Table \ref{tb:OnlineM}, Online detection algorithm demonstrates high detection power and low false alarm rate for high-dimensional data.

\begin{table}[t] 
\centering{
  \caption{Online detection power for mean change of $\Delta = 1/ \sqrt[3]{d}$,  and variance change of $\Sigma = 2 I_{d}$.  $d$ is the dimension of the data. Detect with multiple scanning window length of $\{20, 35,50\}$, total significance is $6\%$ }\label{tb:OnlineM}
\begin{tabular}{ccrrrrr}
 \toprule 
  &d           & 1    & 10    & 100  & 300  & 500  \\
 \toprule 
$\Delta$ &P\_mean  & 0.98 & 0.98  & 0.99 & 0.99 & 0.98 \\
& FPR    & 0.04 & 0.06  & 0.06 & 0.05 & 0.06 \\
\cmidrule(r){1-7}  
 $\Sigma$ & P\_mean & 0.63 & 0.98 & 0.99 & 0.99 & 0. 99\\
& FPR    & 0.06 & 0.06  & 0. 06& 0.03 & 0.05 \\
\bottomrule
\end{tabular}
}
\end{table}

\subsection{Application to S\&P 500 stocks}
By applying the online algorithm to real data, we study the daily closing stock data for companies listed in the S\&P500 from January 2014 to January 2016. The data is log-return of the stock price. Typically, there are 253 open days in a year. In financial market, the month-to-month, quarter-to-quarter change is generally reported on a regular basis. Therefore, we apply window length $n=32$ to capture the quarterly changes. By adjusting the level of significance $\alpha$, we can adjust the rate of false of alarm. In Figure \ref{fig:SP500}, with equal variance adjustment, major events associated with market were identified. Few change of mean detected in August, 2015 which is the three-day of market drop (7.7\% in DJIA). It was reported linking to Greek debt default in June 2015, and Chinese stock market turbulence in July. In early 2016, there are several change of mean detected, which are reported associated to a sharp rise in bond yields in early 2016. The variance change in Figure \ref{fig:SP500}, shows the market volatility is changing often compared to mean change. In practice, the change of variance can be a market volatility risk indicator.   
\begin{figure}[h]
  \includegraphics[width=.90\textwidth]{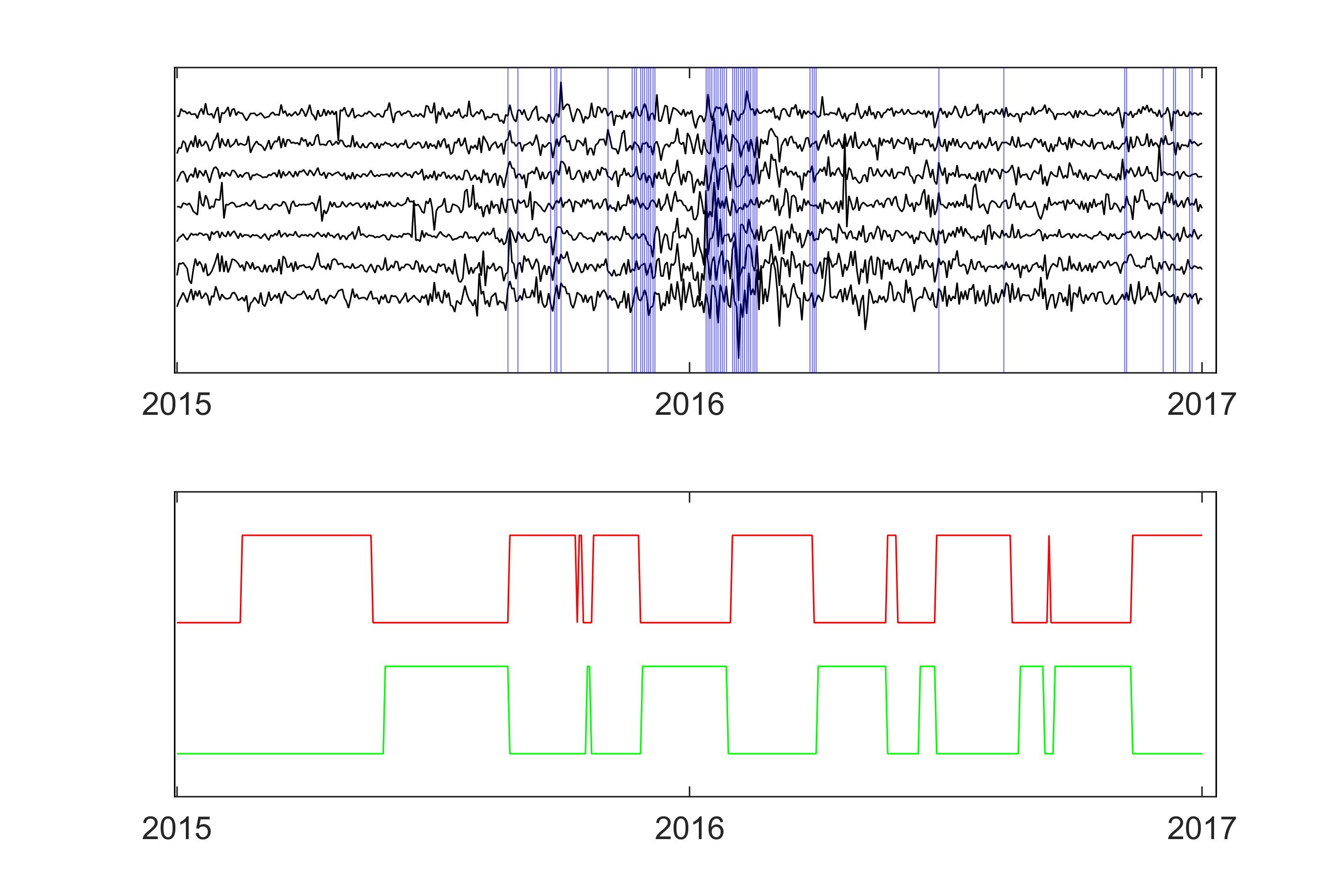}\setlength{\belowcaptionskip}{-5pt}
  \caption{Online detection on daily closing price of S\&P 500 stocks from 2015 to 2017 with window length of $n=32$. For demonstration purposes, 7 out of the 500 stocks are presented in this plot. The blue lines on the upper figure are the change of mean detected. On the lower figure, the red/green line shows the increase/decrease in variance.}\label{fig:SP500}
\end{figure}
\section*{Conclusion}
We have presented a new graph-spanning ratio algorithm for change-point problem for data from low-to high- dimension. Comparing to a recent literature, our approach is sensitive to both mean and variance change.  With thoroughly theoretical study, we have shown that graph-spanning ratios map the high-dimensional Gaussian distributed data into Fisher distributed metrics, based on which, we can test if a change mean or variance has occurred. We show that error probability $(\alpha,\beta)$ can be achieved when the mean residual spanning distance exceed the lower bound of the minimax separation rate. Numerical studies show that the method has desirable power with small and multiple scanning windows, which enables online timely detection of change-point. We conclude with application to real S\&P500 data from financial industry to make statistical inferences about mean and variance changes in 500 stocks.

\newpage

\appendix

\begin{supplement}
\stitle{}
\sdescription{The supplementary materials contain theorems, detailed proofs of the results and numerical analysis in addition to the main paper}
\end{supplement}

\section{Property of graph spanning ratios}

The conditions remain the same as stated in assumption \ref{iid},  \ref{gaussion}, and \ref{homovar}.
	
	\begin{lemma} \label{sp_G}
		The distance spanned by a complete graph with $n$ nodes,  $ G_{n} $ 
		\[ 
		    \| W_{{G}_{n}}\|^2  \sim n \sigma^2 \cdot \chi^2_{(n-1)d}
		 \]
	where $ \| W_{{G}_{n}}\|^2$	is $\chi^2$ distributed with $(n-1)d$ degrees of freedom.  $\|\cdot\|$ is the euclidean distance in $ \R $. 
	\end{lemma}
	
	\begin{proof}
		Without loss of generality, first we assume $d=1$,  and $\sigma =1$. That is $Y_i \in R$, $i \in \{1, \cdots n \}$,  $ (Y_i - Y_j) \sim \mathcal{N}(0, 2) $ for $ i \neq j $ and $ i, j = 1,2, \dots, n $.  Using the fact that the graph is fully connected ,  we know $\| W_{{G}_{n}}\|^2$ is a quadratic sum of  $(Y_i - Y_j)$. We can write $ \| W_{{G}_{n}}\|^2$ as a vector $\mathbf{Y}$,  $ \mathbf{Y} \in R^{n \times 1}$ times a symmetric metric $\mathbf{A}_{G_n} $,  a $n \times n$ matrix. 
		\[  \| W_{{G}_{n}}\|^2 =\sum_{i = i}^{n}  \sum_{j=1}^{i-1} \|Y_i -Y_j\|^2 = \mathbf{Y}^{T} \mathbf{A}_{G_n} \mathbf{Y} \]
		Specifically,  vector $\mathbf{Y}$ and the symmetric matrix $\mathbf{A}_{G_n}$ are \\
		\begin{equation}\label{eq:matrixA}
	     \mathbf{Y} =
		 \begin{bmatrix}
		 Y_1 \\
		 Y_2 \\
		 \vdots\\
		 Y_{n-1}\\
		 Y_n
		 \end{bmatrix},		
		\mathbf{A}_{G_n}=
\begin{bmatrix}
(n-1) & -1 & -1 & \cdots & -1 \\
 -1 & (n-1) & -1 & \cdots & -1 \\
 \vdots & \vdots & \ddots &  & \vdots \\
  -1 & -1 & \cdots & (n-1) &-1   \\
 -1 & -1 & \cdots & -1 & (n-1)
\end{bmatrix}
		\end{equation}
$\mathbf{A}_{G_n}$ is a real and symmetric matrix,  then by spectral decomposition theorem,  $\mathbf{A}_{G_n}$ can be expressed as the multiplication of an orthonormal matrix $\mathbf{Q}$ with a real and diagonal matrix $\mathbf{\Lambda}$,
\[ \mathbf{A}_{G_n}  =\mathbf{Q \Lambda Q}^T\].
The eigenvectors of $\mathbf{A}_{G_n}$ can be chosen to be orthonormal with corresponding real eigenvalues.
Thus,  $ \| W_{{G}_{n}}\|^2$ can be written as a linear combination of independent chi-squared random variable.

\[  \| W_{{G}_{n}}\|^2 = \mathbf{Y}^{T} \mathbf{A}_{G_n} \mathbf{Y} = \mathbf{Y}^{T}  \mathbf{Q \Lambda \mathbf{Q}^T\mathbf{Y} =  \mathbf{U}^T \Lambda} \mathbf{ U} =  \sum^{n}_{i=1} \lambda_i U^2_i  \]

where $ \mathbf{U} = \mathbf{Q}^T \mathbf{Y} $ and $ \mathbf{U}$ is multivariate normal with mean zero and identity covariance matrix,  $U ^2_i \sim \chi^2_1$. $ {\lambda}_i$, $ i \in \{ 1, \cdots, n \} $ are the diagonal element of the diagonal matrix $\mathbf{\Lambda}$.  The interested reader could refer to Chapter 4 of \cite{mathai1992quadratic} for quadratic representations of multivariate normal distributions. 

To show the exact distribution of $ \| W_{{G}_{n}}\|^2$, we need to find the eigenvalue of the matrix $\mathbf{A}_{G_n}$.  

Rewrite Equation (\ref{eq:matrixA}), we obtain the following relation.
\[ 
  \mathbf{A}_{G_n} = - \mathbf{1}_{nxn} + n \mathbf{I}_{n \times n} 
\]

 the eigenvalues of $\mathbf{A}_{G_n}$ are the eigenvalues of $ - \mathbf{1}_{n \times n} + n$.

$tr(- \mathbf{1}_{n \times n})=-n$ and $rank(- \mathbf{1}_{n \times n})=1$.  Therefore, the eigenvalue of $ - \mathbf{1}_{n \times n} $ is $-n, 0, \cdots,0$ (with $(n-1)'s$ $0$).
Thus, the eigenvalues of $\mathbf{A}_{G_n} $ are $0, n, \cdots,n$ (with $(n-1)$'s $n$). \\

Henceforth,  

		\[ 
		    \| W_{{G}_{n}}\|^2 = n \sum^{(n-1)}_{i=1} U^2_i \sim n \cdot \chi^2_{(n-1)}
		 \]

		To extend to dimension $d > 1$:  By the fact that $Y_i \sim \mathcal{N}(\mu,  \sigma^{2})$ and  $ \| W_{{G}_{n}}\|^2$ is a sum of quadratic distances between nodes, adding the quadratic distances from each dimension is simply adding $d$ degrees of freedom. Thus,  we can multiply the degree of freedom by $d$.  To generalize to normal distribution with variance of $\sigma^2$, we multiply the $\chi^2$ with $\sigma^2$. 
		Then
		\[ \| W_{{G}_{n}}\|^2 = n \sigma^2 \sum^{(n-1)d}_{i=1} U^2_i \sim n \sigma^2 \cdot \chi^2_{(n-1)d},  \]
		 and the proof now is complete..
\end{proof}

		\begin{lemma} \label{sp_rem2}
		Let $\| W_{rem,n}\|^2 =  \| W_{G_{2n}}\|^2 -  2 (\|W_{G^{l}_{n}}\|^2 - \|W_{G^{r}_{n}}\|^2 ) $, then
		\[ 
		 \| W_{rem,n}\|^2  \sim 2 n \sigma^2  \cdot  \chi^2_d
		 \]
	\end{lemma}

\begin{proof}
Similarly to Lemma \ref{sp_G},  first let $d = 1$ and let $Y_i$ be standard normal distributed.  We decompose the remainder term as

\begin{align*}
\|W_{rem,n}\|^2 &=  \| W_{g_{2n}}\|^2 -  2 (\|W_{G^{l}_{n}}\|^2 - \|W_{G^{r}_{n}}\|^2 =  \mathbf{Y}^T ( \mathbf{A}_{G_2n} - (\mathbf{A}_{G^l_n} +\mathbf{A}_{G^r_n}  )) \mathbf{Y} \\
&=  \mathbf{Y}^T\mathbf{A}_{rem} \mathbf{Y} 
\end{align*} 
where 
\[
	\mathbf{Y} =
		 \begin{bmatrix}
		 Y_1 \\
		 Y_2 \\
		 \vdots\\
		 Y_{2n-1}\\
		 Y_{2n}
		 \end{bmatrix},		
	\mathbf{A}_{G_{2n}} =
	\begin{bmatrix}
	(2n-1) & -1 & -1 & \cdots & -1 \\
	 -1 & (2n-1) & -1 & \cdots & -1 \\
	 \vdots & \vdots & \ddots &  & \vdots \\
 	 -1 & -1 & \cdots & (2n-1) &-1   \\
 	-1 & -1 & \cdots & -1 & (2n-1)
	\end{bmatrix}
\]
\begin{align}\label{A_rl}
	2(\mathbf{A}_{G^l_{n}}+\mathbf{A}_{G^l_{n}}) &= 2 
	\begin{bmatrix}
	(n-1) & -1  & \cdots & -1 & 0 & 0 & \cdots &0 \\
	 \vdots & \ddots &  &  \vdots & \vdots  & \vdots & \ddots & \vdots \\
	 -1 & -1 & \cdots & (n-1) & 0 & 0 & \cdots & 0  \\
 	 0 & 0 & \cdots & 0 &(n-1) & -1  & \cdots & -1    \\
 	  \vdots & \ddots &  &  \vdots & \vdots  & \ddots & & \vdots \\
 	 0 & 0 & \cdots & 0 &-1 & -1 & \cdots & (n-1) 
	\end{bmatrix} \\
	&= 2
	\begin{bmatrix}
	\mathbf{A}_{G_n} & \mathbf{0}_{n \times n} \\
	\mathbf{0}_{n \times n} & \mathbf{A}_{G_n} 
	\end{bmatrix}
\end{align}

Thus, 
	\begin{equation}\label{A_rem}
	\mathbf{A}_{rem}= 
	\begin{bmatrix}
	\mathbf{1}_{n \times n} & - \mathbf{1}_{n \times n} \\
	-\mathbf{1}_{n \times n} & \mathbf{1}_{n \times n} 
	\end{bmatrix}
	\end{equation}
	Since $tr(\mathbf{A}_{rem})=2n$ and $rank(\mathbf{A}_{rem})=1$.  Therefore, the eigenvalue of $ \mathbf{A}_{rem}$ is $2n, 0, \cdots,0$ (with $(2n-1)'s$ $0$).
Thus, $\| W_{rem,n}\|^2 =  \| W_{G_{2n}}\|^2 -  2 (\|W_{G^{l}_{n}}\|^2 - \|W_{G^{r}_{n}}\|^2) \sim 2n \sigma^2 \cdot \chi^2_{(n-1)d}$ is proved.
	\end{proof}
	
\textbf{Proof of Lemma \ref{sp_idep} }
\begin{proof}
$\| W_{G^l_n}\|^2 + W_{G^r_n}\|^2=  \mathbf{Y}^T(\mathbf{A}_{G^l_n} +\mathbf{A}_{G^r_n} )\mathbf{Y} $ and $\| W_{rem,n}\|^2 =  \mathbf{Y}^T\mathbf{A}_{rem} \mathbf{Y} $.
By Equation (\ref{A_rl}) and Equation (\ref{A_rem}), 
\begin{align*}
 (\mathbf{A}_{G^l_{n}}+\mathbf{A}_{G^l_{n}}) \cdot
 \mathbf{A}_{rem} &=
\begin{bmatrix}
	\mathbf{A}_{G_n} & \mathbf{0}_{n \times n} \\
	\mathbf{0}_{n \times n} & \mathbf{A}_{G_n} 
	\end{bmatrix}
	\cdot{}
		\begin{bmatrix}
	\mathbf{1}_{n \times n} & - \mathbf{1}_{n \times n} \\
	-\mathbf{1}_{n \times n} & \mathbf{1}_{n \times n} 
	\end{bmatrix}  \\
	&=\begin{bmatrix}
	\mathbf{A}_{G_n} \cdot \mathbf{1}_{n \times n}  & - \mathbf{A}_{G_n} \cdot  \mathbf{1}_{n \times n} \\
	-\mathbf{A}_{G_n} \cdot \mathbf{1}_{n \times n}  & \mathbf{A}_{G_n} \cdot  \mathbf{1}_{n \times n}
	\end{bmatrix}. 
\end{align*}

Since $\mathbf{A}_{G_n} \cdot \mathbf{1}_{n \times n} = 0$,
\[ (\mathbf{A}_{G^l_{n}}+\mathbf{A}_{G^r_{n}}) \cdot  \mathbf{A}_{rem} = \mathbf{0}_{2n \times 2n}\]
$ (\mathbf{A}_{G^l_{n}}+\mathbf{A}_{G^r_{n}})$ and $ \mathbf{A}_{rem} $ are linearly independent. \\
Henceforth,  $ \| W_{rem,n}\|^2 $ and $(\|W_{G^{l}_{n}}\|^2 + \|W_{G^{r}_{n}}\|^2)$ are independent.
\end{proof}

		\begin{corollary} \label{sp_rem}
		Under Assumption \ref{iid} and \ref{gaussion},  the expected distance spanned in-between two complete graph $ G^{l}_{n} $ and  $ G^{r}_{n} $ is
		\[ 
		     \mathbf{E}( \| W_{btw,n} \|^2) = 2\sigma^2 n^2d 
		 \]
	\end{corollary}
	
	\begin{proof}
	By Lemma \ref{sp_G},  Lemma \ref{sp_rem2},  and Lemma\ref{sp_idep} , $\|W_{btw,n}\|^2 = $ , is a combination of two independent $\chi^2$-distributed random variables. Therefore, we can have the following relation,
	 
	  \[\|W_{btw,n}\|^2 = \sigma^2 n \sum^{2(n-1)d}_{i=1} U^2_i + \sigma^2 2n  \sum^{2n \cdot d}_{i=2(n-1)d+1} U^2_i.\]
	  We take expected value of the sum of the $\chi^2$ random variables and then the result follows.
	\end{proof}
Next, we study some theoretical properties of the spanning distance of the graph and the distribution of the GSR test statistics.

\textbf{Proof of Proposition \ref{lemma:testmu}}

\begin{proof}
By Proposition (\ref{sp_G}) and Proposition (\ref{sp_rem2}),  $\|W_{G^{rem}_{2n}}\|^2$ and $ \| W_{G^{l}_{n}} \|^2$, $ \| W_{G^{r}_{n} \|}^2$,  follow $\chi^2 $ distribution and are linearly independent and thus the result follows.  
\begin{align*}
R_{\mu,n} - 2 &= \frac{\| W_{{G}_{2n}}\|^2 - 2 (\| W_{G^{l}_{n}} \|^2+\| W_{G^{r}_{n}}\|^2) }{\| W_{G^{l}_{n}} \|^2+\| W_{G^{r}_{n}}\|^2} \\
&=  \frac{\| W_{rem,n}\|^2}{\| W_{G^{l}_{n}} \|^2+\| W_{G^{r}_{n}}\|^2} \\
& \sim \frac{1}{(n-1)} F_{d, 2(n-1)d}. 
\end{align*}
\end{proof}
	
\textbf{Proof of Proposition \ref{proposition:fisher1}}	
\begin{proof}
By definition of the spanning distance:
	\[ \| W_{G^{l}_{n,t}} \|^2 = \sum_{i =t-n}^{t-1}  \sum_{i =t-n}^{i-1}  \|Y_i -Y_j\|^2, \]
    \[ \| W_{G^{r}_{n,t}}\|^2 = \sum_{i =t}^{t+n-1}   \sum_{i =t}^{i-1}    \|Y_i -Y_j\|^2. \]	 
The two random variables $\|W_{G^{l}_{n}} \|^2$ and $\|W_{G^{r}_{n}} \|^2$ are $\chi^2$ distributed and they are independent to each other. Then, by Lemma \ref{sp_G}, both metrics follow Fisher distribution with degrees of freedom $(n-1)d$ and $(n-1)d$. 
\end{proof}

\begin{corollary}
The testing threshold of $\alpha$-level significance for graphical mean is 
\[\rho_{\mu,n}(\alpha) = \frac{1}{(n-1)}F^{-1}_{d,2(n-1)d} (\alpha) + 2 \]
\end{corollary}

\begin{corollary}
The testing thresholds of $\alpha$-level significance for graphical variance are is 
\[\rho_{\sigma+,n}(\alpha) =\rho_{\sigma-,n}(\alpha)  = \frac{1} {(n-1)}F^{-1}_{(n-1)d,(n-1)d} (\alpha)  \]
\end{corollary}

\textbf{Proof of Theorem \ref{Theo:Delta_mu}}


\begin{proof}
\[  \mathds{T}_{\mu}=\sup_{n \in \mathfrak{N}} \Bigg\{ \frac{ \|W_{G_{2n}}\|^2 }{\Big( \|W_{G^{l}_{n}}\|^2 + \|W_{G^{r}_{n}}\|^2 \Big)}   - 2\frac{N_n }{D_n} F^{-1}_{N_n, D_n} (\alpha_{\mu,n}) \Bigg\} \]

By definition of $ \mathds{T}_{\mu}$,  $\Pm(  \mathds{T}_{\mu} \le 0) \le \inf_{n \in \mathfrak{N}} P(n)$ where 

\begin{align*}
 P(n) &= \Pm \Bigg(\frac{ \|W_{G_{2n}}\|^2 }{\Big( \|W_{G^{l}_{n}}\|^2 + \|W_{G^{r}_{n}}\|^2 \Big)}  \le 2\frac{N_n }{D_n}F^{-1}_{N_n,D_n}(\alpha_{\mu,n}) \Bigg) \\
 &= \Pm \Bigg( \frac{ \|W_{rem,n}\|^2 }{\Big( \|W_{G^{l}_{n}}\|^2 + \|W_{G^{r}_{n}}\|^2 \Big)}  \le 2\frac{N_n }{D_n} F^{-1}_{D_n,N_n}(\alpha_{\mu,n}) -2  \Bigg) 
\end{align*}

The goal is to show $P(n) \le \beta$. 

Denote $Q (a, D, u)$ the $1-u$ quantile of a non-central $\chi^2$ random variable with $D$ degree of freedom and non-centrality parameter $a$.

For each  $n  \in \mathfrak{N}$, we have

\[ 
	\| W_{G^{l}_{n}}\|^{2} +  \| W_{G^{r}_{n}}\|^{2}  \sim \chi^2_{D_n}
\]
with non-centrality parameter $ \|\mu_{G^{r}_n}\|^2+\|\mu_{G^{l}_n}\|^2$,  degree of freedom $D_n = 2(n-1) d$, and 

\[ 
	\|W_{rem,n}\|^{2} \sim 2\chi^2_{N_n}
\]
with non-centrality parameter  $\|\mu_{rem,n}\|^2$, degree of freedom $N_n = d $.

Note that the mean spanning distance for graph $G_{2n} $ under $H_0$ is

\begin{equation}\label{eq:mu}
\|\mu_{G_{2n}}\|^2 := E_0 [\| W_{G_{2n}}\|^{2}] =2( \|\mu_{G^{r}_n}\|^2+\|\mu_{G^{l}_n}\|^2)+\|\mu_{rem,n}\|^2 
 \end{equation}

by Proposition \ref{lemma:testmu},  
 
\[
R_{\mu, n} = \frac{ \|W_{G_{2n}}\|^{2}}{ (\|W_{G^{l}_{n}}\|^{2} +  \|W_{G^{r}_{n}}\|^{2})}\ -2 \sim 2\frac{N_n }{D_n} F_{N_n, D_n},
\]

Thus, the test-statistics $R_{\mu,n}$ follows Fisher distribution with $N_n$ and $D_n$ degrees of freedom. Hence, 
\begin{align*}
 P(n) &= \Pm \Bigg( \frac{  \|W_{G^{rem}_{2n}}\|^2 }{\Big( \|W_{G^{l}_{n}}\|^2 + \|W_{G^{r}_{n}}\|^2 \Big)}  \le 2\frac{N_n }{D_n} F^{-1}_{N_n,D_n}(\alpha_{\mu,n})   \Bigg)  \\
 &\le \Pm \Bigg(  \|W_{G^{rem}_{2n}}\|^{2} \leq  2\frac{N_n}{D_n} F^{-1}_{N_n, D_n}(\alpha_{\mu,n}) Q \Bigg(  \| \mu^{l}_{G_n} \|^2+\| \mu^{r}_{G_n} \|^2, D_n, \frac{\beta}{2}\Bigg)\Bigg) + \frac{\beta}{2}
\end{align*}

Therefore, 
\[
P( \mathds{T}_{\mu}  \le 0) \le \beta 
\]

if for some $n$ in $\mathfrak{N}$
\begin{equation}\label{eq:quotion}
2\frac{N_n}{D_n} F^{-1}_{N_n, D_n} (\alpha_{\mu,n}) Q\Bigg(\|\mu^{l}_{G_n}\|^2+\|\mu^{r}_{G_n}\|^2, D_n, \frac{\beta}{2}\Bigg) \le Q\Bigg(\|\mu_{rem,n}\|^2, N_n, 1-\frac{\beta}{2}\Bigg)
\end{equation}

By Lemma 3 from Birge(2001), we obtain

\[
Q(a, D, u) \le D+a+2 \sqrt{(D+2a) \log(1/u)} + 2 \log (1/u)
\]

\[
Q(a, D, 1-u) \geq  D+a-2 \sqrt{(D+2a) \log(1/u)}
\]
therefore

\begin{align*}
Q \Bigg( \| \mu^{l}_{G_n} \|^2+\| \mu^{r}_{G_n} \|^2, D_n, \frac{\beta}{2}\Bigg)
& \le D_n+ ( \| \mu^{l}_{G_n} \|^2+\| \mu^{r}_{G_n} \|^2) \\
& \;  \; \; \;+ 2 \sqrt{(D_n +2(  \| \mu^{l}_{G_n} \|^2+\| \mu^{r}_{G_n} \|^2) \log(2/\beta)}  \\
& \;  \; \; \;+ 2 \log(2/\beta) \\
&= D_n +(\| \mu^{l}_{G_n} \|^2+\| \mu^{r}_{G_n} \|^2) \\
& \;  \; \; \;+2 \sqrt{D_n \log(2/\beta)+2 (  \| \mu^{l}_{G_n} \|^2+\| \mu^{r}_{G_n} \|^2) \log(2/\beta)} \\
& \;  \; \; \; +2 \log(2/\beta) \\
\end{align*}
By the inequality $\sqrt{u+v} \le \sqrt{u} + \sqrt{v}$, and $2 \sqrt{uv} \le 1/2u+2v$ 

\begin{align*}
Q &\Bigg(  \| \mu^{l}_{G_n} \|^2+\| \mu^{r}_{G_n} \|^2, D_n, \frac{\beta}{2}\Bigg)  \\
 &\le  D_n +  (\| \mu^{l}_{G_n} \|^2+\| \mu^{r}_{G_n} \|^2) 
 + 2 \sqrt{D_n \log \Big(\frac{2}{\beta}\Big)} + 2 \sqrt{(\| \mu^{l}_{G_n} \|^2+\| \mu^{r}_{G_n} \|^2) 2 \log\Big(\frac{2}{\beta}\Big)} \\
\end{align*}

\begin{equation}\label{eq:4log}
\le  D_n + 2(\| \mu^{l}_{G_n} \|^2+\| \mu^{r}_{G_n} \|^2) + 2 \sqrt{D_n \log \Big(\frac{2}{\beta}\Big)} +4 \log \Big(\frac{2}{\beta}\Big)
\end{equation}

\[
Q(a, D, 1-u) \geq D+a-2 \sqrt{(D+1a)\log(2/ \beta)} \\
\]
We obtain

\begin{align*}
Q\Bigg(\|\mu_{rem,n}\|^2, N_n, 1-\frac{2}{\beta}\Bigg) &  \geq N_n + \|\mu_{rem,n}\|^2 - 2\sqrt{\Big(N_n + 2 \|\mu_{rem,n}\|^2 \Big) \log\Big(\frac{2}{\beta}\Big)} \\
\end{align*}

By inequality $ \sqrt{u+ v} \leq \sqrt{u} + \sqrt{v} $
\[
\geq N_n + \|\mu_{rem,n}\|^2 - 2 \sqrt{N_n \log \Big(\frac{2}{\beta}\Big)}-2 \sqrt{2 \|\mu_{rem,n}\|^2 \log \Big(\frac{2}{\beta}\Big)}
\]

By inequality $2\sqrt{u v} \leq \theta u + \theta^{-1} v $, choose $\theta = \frac{1}{5}$
\[ 
\geq N_n +\|\mu_{rem,n}\|^2 - 2 \sqrt{N_n \log\Big(\frac{2}{\beta}\Big)} - \frac{1}{5} \|\mu_{rem,n}\|^2 - 10 \log\Big(\frac{2}{\beta}\Big)
\]

\begin{equation}\label{eq:10log}
= N_n +\frac{4}{5} \|\mu_{rem,n}\|^2 - 2 \sqrt{N_n \log(\frac{2}{\beta})} - 10 \log \Big(\frac{2}{\beta}\Big)
\end{equation}

Based on Equation (\ref{eq:quotion},) we have the following relation satisfied 
\[
2\frac{N_n}{D_n} F^{-1}_{N_n, D_n} (\alpha_{\mu,n}) Q\Bigg(\| \mu^{l}_{G_n} \|^2+\| \mu^{r}_{G_n} \|^2, D_n, \frac{\beta}{2}\Bigg) \le Q \Bigg(\|\mu_{rem,n}\|^2, N_n, 1-\frac{\beta}{2}\Bigg).
\]

Plug Equation (\ref{eq:4log}) and Equation (\ref{eq:10log}) into Equation (\ref{eq:quotion}), we obtain the following relation:

\begin{align*}
2\frac{N_n}{D_n} F^{-1}_{N_n, D_n} (\alpha_{\mu,n})  \Bigg( D_n + 2( \| \mu^{l}_{G_n} \|^2+\| \mu^{r}_{G_n} \|^2) + 2 \sqrt{D_n \log\Big (\frac{2}{\beta}\Big)} +4 \log \Big(\frac{2}{\beta}\Big) \Bigg)  \\
\leq N_n + \frac{4}{5} \|\mu_{rem,n}\|^2 - 2 \sqrt{D_n \ log\Big(\frac{2}{\beta}\Big)} - 10 \log \Big(\frac{2}{\beta}\Big)
\end{align*}

\begin{align*}
 \frac{5}{2}\frac{N_n}{D_n}  F^{-1}_{N_n, D_n} (\alpha_{\mu,n}) \Bigg( D_n + 2 \Big( \| \mu^{l}_{G_n} \|^2+\| \mu^{r}_{G_n} \|^2 \Big) + 2 \sqrt{D_n \log (\frac{2}{\beta})} +4 \log \Big(\frac{2}{\beta}\Big) \Bigg)  \\
\leq  \frac{5}{4} \Bigg( N_n  - 2 \sqrt{N_n \ log \Big(\frac{2}{\beta}\Big)} - 10 \log \Big (\frac{2}{\beta} \Big) \Bigg)+ \|\mu_{rem,n}\|^2
\end{align*}

Rearrange the equation, we have

\begin{align*}
\|\mu_{rem,n}\|^2 \geq  &\Bigg( 5 \frac{N_n}{D_n} F^{-1}_{N_n, D_n} (\alpha_{\mu,n}) \Bigg) \Bigg(\| \mu^{l}_{G_n} \|^2+\| \mu^{r}_{G_n} \|^2 \Bigg) \\
&+\Bigg( 5\frac{N_n}{D_n} F^{-1}_{N_n, D_n} (\alpha_{\mu,n}) \Bigg)\Bigg( D_n + 2 \sqrt{D_n \log\Big(\frac{2}{\beta}\Big)} +4 \log\Big(\frac{2}{\beta}\Big) \Bigg) \\
&-  \frac{5}{4} \Bigg(N_n  - 2 \sqrt{N_n \ log\Big(\frac{2}{\beta}\Big)} - 10 \log \Big (\frac{2}{\beta}\Big) \Bigg)
\end{align*} 

Apply Equation (\ref{eq:mu}) we have derived the quantity
\begin{align*}
\|\mu_{G_{2n}}\|^2 \geq & \Bigg( 2 + 5 \frac{N_n}{D_n} F^{-1}_{N_n, D_n} (\alpha_{\mu,n}) \Bigg) \Bigg(\| \mu^{l}_{G_n} \|^2+\| \mu^{r}_{G_n} \|^2 \Bigg) \\
&+\Bigg( 5\frac{N_n}{D_n} F^{-1}_{N_n, D_n} (\alpha_{\mu,n}) \Bigg)\Bigg( D_n + 2 \sqrt{D_n \log\Big(\frac{2}{\beta}\Big)} +4 \log\Big(\frac{2}{\beta}\Big) \Bigg) \\
&-  \frac{5}{4} \Bigg(N_n  - 2 \sqrt{N_n \ log\Big(\frac{2}{\beta}\Big)} - 10 \log \Big (\frac{2}{\beta}\Big) \Bigg)
\end{align*}

\end{proof}

	We define	$\| W_{btw,n} \|^2 $ as the sum of quadratic distance between nodes of $G^{l}_{n}$ and $G^{r}_{n}$. This metric tells the separation gap between data before and after the change-point candidate.
		
\begin{align*}
\| W_{btw,n} \|^2 &=\|W_{G_{2n}}\|^2 - \|W_{G^{l}_{n}}\|^2 - \|W_{G^{r}_{n}}\|^2 \\
&= \sum\limits_{ i \in G^{l}_{n},  j \in G^{r}_{n}}\|Y_i -Y_j\|^{2}
\end{align*}
Denote$  \| \mu_{btw}\|^2 = \mathbf{E}( \| W_{btw,n} \|^2)$.

		\begin{corollary} \label{sp_btw}
		Under Assumption \ref{iid} and \ref{gaussion},  define the expected distance spanned in-between two complete graph $ G^{l}_{n} $ and  $ G^{r}_{n} $ as
		\[ 
		     \| \mu_{btw}\|^2 = 2\sigma^2 n^2d 
		 \]
	\end{corollary}

\begin{corollary}
 Given window size $n$,  $ \Pm (  T_{\mu,n} > 0) \ge 1-\beta$, if
 \[  \| \mu_{btw} \| ^2 > (C_1-1) \Big( \| \mu_{l,n}\|^2+\| \mu_{r,n} \|^2 \Big) + C_2 \sigma^2, \]
where $\| \mu_{btw} \| ^2  \in \mathbb{R}$ is the mean separation spanning distance between graph $G^{l}_n$ and  $G^{r}_n$.
\end{corollary}
From other prospective, if the separation between before- and after-graphs is greater than the described quantity,  then the detecting power of $1-\beta$ is guaranteed.

\begin{prop}[Power of the test- $\sigma+$]
Let $ T_{\sigma+,n}$ be the test statistics specified in the main paper and $\beta \in (0,1)$. Then, for any given fixed window length $n$, $ \Pm (T_{\sigma+,n} > 0) \ge 1-\beta$, if 
\[  \| \mu_{r,n} \| ^2 \geq  C_1 \Big( \| \mu_{l,n} \|^2 \Big) + C_2 \sigma^2,\]
where $\| \mu_{r,n} \| ^2  \in \mathbb{R}$ is the mean spanning distance of graph $G^{r}_n$,  and
\begin{align*}
 C_1 =& \Bigg( \frac{5}{2} \frac{d^{r}_{n}}{d^{l}_{n}} F^{-1}_{d^{r}_{n}, d^{l}_{n}} (\alpha_{\sigma+,n}) \Bigg), \\
 C_2 =& \frac{5}{4}\frac{d^{r}_{n}}{d^{l}_{n}}  \Bigg( F^{-1}_{d^{r}_{n}, d^{l}_{n}} (\alpha_{\sigma+,n}) \Bigg) \\
 & \Bigg( d^{l}_{n} + 2 \sqrt{d^{l}_{n}\log\Big(\frac{2}{\beta}\Big)} +4 \log\Big(\frac{2}{\beta}\Big) \Bigg) \\
& -\frac{5}{4} \Bigg(d^{r}_{n}  - 2 \sqrt{d^{r}_{n}\ log \Big(\frac{2}{\beta}\Big)} - 10 \log \Big(\frac{2}{\beta}\Big) \Bigg), 
\end{align*}
where $d^{r}_{n} = d^{l}_{n} = (n-1)d$.
\end{prop}\label{prop: vari+}
\begin{proof}

By definition of $ \mathds{T}_{\sigma+, n}$.  Let
\begin{align*}
 P(n) = \Pm(  \mathds{T}_{\sigma+,n} \le 0)) &= \Pm \Bigg(\frac{ \|W_{G^r_{n}}\|^2 }{ \|W_{G^{l}_{n}}\|^2 }  \le  \rho_{\sigma+,n}(\alpha_{\sigma+,n}) \Bigg) \\
\end{align*}

The goal is to show $P(n) \le \beta$. 

Denote $Q (a, D, u)$ the $1-u$ quantile of a non-central $\chi^2$ random variable with $D$ degree of freedom and non-centrality parameter $a$.

For each  $n  \in \mathfrak{N}$, we have

\[ 
	\| W_{G^{l}_{n}}\|^{2}   \sim \chi^2_{d^l_n}
\]
with non-centrality parameter $ \|\mu_{G^{l}_n}\|^2$, and 

\[ 
	\| W_{G^{r}_{n}}\|^{2}   \sim \chi^2_{d^r_n}
\]
with non-centrality parameter  $\|\mu_{G^{r}_n}\|^2$

by Proposition \ref{prop:var},  
 
\[
R_{\sigma+, n} = \frac{ \|W_{G^{r}_{n}}\|^{2}}{ \|W_{G^{l}_{n}}\|^{2}} \sim \frac{d^r_n }{dl_n} F_{d^r_n, d^l_n},
\]

Thus, the test-statistics $R_{\sigma+,n}$ follows Fisher distribution with $d^r_n$ and $d^l_n$ degrees of freedom. Hence, 
\begin{align*}
 P(n) &= \Pm \Bigg( \frac{ \|W_{G^{r}_{n}}\|^{2}}{ \|W_{G^{l}_{n}}\|^{2}}  \le \frac{d^r_n }{dl_n} F^{-1}_{d^r_n, d^l_n} (\alpha_{\sigma+,n}) \Bigg)  \\
 &\le \Pm \Bigg(  \|W_{G^{r}_{n}}\|^{2} \leq  \frac{d^r_n }{dl_n} F^{-1}_{d^r_n, d^l_n}(\alpha_{\sigma+,n}) Q \Bigg(  \| \mu^{l}_{G_n} \|^2 , d^l_n, \frac{\beta}{2}\Bigg)\Bigg) + \frac{\beta}{2}
\end{align*}

Therefore, 
\[
P( \mathds{T}_{\mu}  \le 0) \le \beta 
\]

if for some $n$ in $\mathfrak{N}$
\begin{equation}\label{eq:quotionS+}
\frac{d^r_n }{dl_n} F^{-1}_{d^r_n, d^l_n}(\alpha_{\sigma+,n}) Q \Bigg(  \| \mu^{l}_{G_n} \|^2 , d^l_n, \frac{\beta}{2}\Bigg)\Bigg) \le Q\Bigg(\|\mu^{r}_{G_{n}}\|^2, d^r_n, 1-\frac{\beta}{2}\Bigg)
\end{equation}

By Lemma 3 from Birge(2001), we obtain

\[
Q(a, D, u) \le D+a+2 \sqrt{(D+2a) \log(1/u)} + 2 \log (1/u)
\]

\[
Q(a, D, 1-u) \geq  D+a-2 \sqrt{(D+2a) \log(1/u)}
\]
therefore

\begin{align*}
Q \Bigg( \| \mu^{l}_{G_n} \|^2,  d^l_n, \frac{\beta}{2}\Bigg)
& \le d^l_n+ ( \| \mu^{l}_{G_n} \|^2) \\
& \;  \; \; \;+ 2 \sqrt{(d^l_n +2(  \| \mu^{l}_{G_n} \|^2) \log(2/\beta)}  \\
& \;  \; \; \;+ 2 \log(2/\beta) \\
&= d^l_n +(\| \mu^{l}_{G_n} \|^2+\| \mu^{r}_{G_n} \|^2) \\
& \;  \; \; \;+2 \sqrt{D_n \log(2/\beta)+2 (  \| \mu^{l}_{G_n} \|^2+\| \mu^{r}_{G_n} \|^2) \log(2/\beta)} \\
& \;  \; \; \; +2 \log(2/\beta) \\
\end{align*}
By the inequality $\sqrt{u+v} \le \sqrt{u} + \sqrt{v}$, and $2 \sqrt{uv} \le 1/2u+2v$ 

\begin{equation*}
Q \Bigg(  \| \mu^{l}_{G_n} \|^2,  d^l_n, \frac{\beta}{2}\Bigg)  
 \le  d^l_n +  (\| \mu^{l}_{G_n} \|^2) + 2 \sqrt{d^l_n \log \Big(\frac{2}{\beta}\Big)} + 2 \sqrt{(\| \mu^{l}_{G_n} \|^2) 2 \log\Big(\frac{2}{\beta}\Big)} \\
\end{equation*}

\begin{equation}\label{eq:4logS}
\le  d^l_n + 2(\| \mu^{l}_{G_n} \|^2) + 2 \sqrt{d^l_n \log \Big(\frac{2}{\beta}\Big)} +4 \log \Big(\frac{2}{\beta}\Big)
\end{equation}

\[
Q(a, D, 1-u) \geq D+a-2 \sqrt{(D+1a)\log(2/ \beta)} \\
\]
We obtain

\begin{align*}
Q\Bigg(\|\mu^{r}_{G_{n}}\|^2, d^r_n, 1-\frac{2}{\beta}\Bigg) &  \geq d^r_n + \|\mu^{r}_{G_{n}}\|^2 - 2\sqrt{\Big(d^r_n + 2 \|\mu^{r}_{G_{n}}\|^2 \Big) \log\Big(\frac{2}{\beta}\Big)} \\
\end{align*}

By inequality $ \sqrt{u+ v} \leq \sqrt{u} + \sqrt{v} $
\[
\geq d^r_n + \|\mu^{r}_{G_{n}}\|^2 - 2 \sqrt{d^r_n \log \Big(\frac{2}{\beta}\Big)}-2 \sqrt{2 \|\mu^{r}_{G_{n}}\|^2 \log \Big(\frac{2}{\beta}\Big)}
\]

By inequality $2\sqrt{u v} \leq \theta u + \theta^{-1} v $, choose $\theta = \frac{1}{5}$
\[ 
\geq d^r_n +\|\mu^{r}_{G_{rn}}\|^2 - 2 \sqrt{d^r_n \log\Big(\frac{2}{\beta}\Big)} - \frac{1}{5} \|\mu^{r}_{G_{n}}\|^2 - 10 \log\Big(\frac{2}{\beta}\Big)
\]

\begin{equation}\label{eq:10logS}
= d^r_n +\frac{4}{5} \|\mu^{r}_{G_{n}}\|^2 - 2 \sqrt{d^r_n \log(\frac{2}{\beta})} - 10 \log \Big(\frac{2}{\beta}\Big)
\end{equation}

Based on Equation (\ref{eq:quotion},) we want to have the following relation satisfied 
\[
\frac{d^r_n}{d^l_n} F^{-1}_{d^r_n, d^l_n} (\alpha_{\sigma+,n}) Q\Bigg(\| \mu^{l}_{G_n} \|^2, d^l_n, \frac{\beta}{2}\Bigg) \le Q \Bigg(\|\mu^{r}_{G_{2}}\|^2, d^r_n, 1-\frac{\beta}{2}\Bigg).
\]

Plug Equation (\ref{eq:4logS}) and Equation (\ref{eq:10logS}) into Equation (\ref{eq:quotionS+}), we obtain the following relation:

\begin{align*}
\frac{d^r_n}{d^l_n} F^{-1}_{d^r_n, d^l_n} (\alpha_{\sigma+,n})  \Bigg( d^l_n + 2( \| \mu^{l}_{G_n} \|^2) + 2 \sqrt{D_n \log\Big (\frac{2}{\beta}\Big)} +4 \log \Big(\frac{2}{\beta}\Big) \Bigg)  \\
\leq d^r_n + \frac{4}{5} \|\mu^{r}_{G_{n}}\|^2 - 2 \sqrt{d^l_n \ log\Big(\frac{2}{\beta}\Big)} - 10 \log \Big(\frac{2}{\beta}\Big)
\end{align*}

\begin{align*}
 \frac{5}{4}\frac{d^r_n}{d^l_n}  F^{-1}_{d^r_n, d^l_n} (\alpha_{\sigma+,n}) \Bigg( d^l_n + 2 \Big( \| \mu^{l}_{G_n} \|^2 \Big) + 2 \sqrt{d^l_n \log (\frac{2}{\beta})} +4 \log \Big(\frac{2}{\beta}\Big) \Bigg)  \\
\leq  \frac{5}{4} \Bigg( d^r_n  - 2 \sqrt{d^r_n \ log \Big(\frac{2}{\beta}\Big)} - 10 \log \Big (\frac{2}{\beta} \Big) \Bigg)+ \|\mu^{r}_{G_{n}}\|^2
\end{align*}

Rearrange the equation, we have derived the quantity 

\begin{align*}
\|\mu^{r}_{G_{n}}\|^2 \geq  &\Bigg( \frac{5}{2} \frac{d^r_n}{d^l_n} F^{-1}_{d^r_n, d^l_n} (\alpha_{\sigma+,n}) \Bigg) \Bigg(\| \mu^{l}_{G_n} \|^2 \Bigg) \\
&+\Bigg( \frac{5}{4}\frac{d^r_n}{d^l_n} F^{-1}_{d^r_n, d^l_n} (\alpha_{\sigma+,n}) \Bigg)\Bigg( d^l_n + 2 \sqrt{d^l_n \log\Big(\frac{2}{\beta}\Big)} +4 \log\Big(\frac{2}{\beta}\Big) \Bigg) \\
&-  \frac{5}{4} \Bigg(d^r_n  - 2 \sqrt{d^r_n \ log\Big(\frac{2}{\beta}\Big)} - 10 \log \Big (\frac{2}{\beta}\Big) \Bigg),
\end{align*} 
where $d^r_n = d^l_n= (n-1)d$ 
\end{proof}

\begin{prop}[Power of the test- $\sigma-$]
Let $ T_{\sigma-,n}$ be the test statistics specified as Equation (5) and $\beta \in (0,1)$. Then, for any given fixed window length $n$, $ \Pm (  T_{\sigma-,n} > 0) \ge 1-\beta$, if 
 
\[ \| \mu_{l,n} \| ^2 \geq  C_1 \Big( \| \mu_{r,n} \|^2 \Big) + C_2 \sigma^2, \]

where $\| \mu_{l,n} \| ^2  \in \mathbb{R}$ is the mean spanning distance of graph $G^{l}_n$, and
\begin{align*}
 C_1 =& \Bigg( \frac{5}{2} \frac{d^{l}_{n}}{d^{r}_{n}} F^{-1}_{d^{l}_{n}, d^{r}_{n}} (\alpha_{\sigma-,n}) \Bigg), \\
 C_2 = & \frac{5}{4}\frac{d^{l}_{n}}{d^{r}_{n}}  \Bigg( F^{-1}_{d^{l}_{n}, d^{r}_{n}} (\alpha_{\sigma-,n}) \Bigg) \\
 & \Bigg( d^{r}_{n} + 2 \sqrt{d^{r}_{n}\log\Big(\frac{2}{\beta}\Big)} +4 \log\Big(\frac{2}{\beta}\Big) \Bigg) \\
& -\frac{5}{4} \Bigg(d^{l}_{n}  - 2 \sqrt{d^{l}_{n}\ log \Big(\frac{2}{\beta}\Big)} - 10 \log \Big(\frac{2}{\beta}\Big) \Bigg). 
\end{align*}
\end{prop}
\begin{proof}
It can be shown in a similar fashion by exchange $\| W^l_{G_n} \|^2$  with $\| W^r_{G_n} \|^2$ in the proof in Proposition \ref{prop: vari+}.
\end{proof}

\subsection{Minimum radius of the mean separation}
We derive the minimal radius, i.e. lower bound of minimax separation rate, based on the result from \cite{baraud2003adaptive}, and \cite{baraud2002non}. 
To measure the performance of the test at a fix window size $n$, we denote a quantity $\rho_n(\mathcal{F}_1,\phi_{\alpha},\delta)$ by
\begin{align*}
\rho_n(\mathcal{F}_1,\phi_{\alpha},\delta) &= \inf \{\rho > 0,  \inf_{ {\mathcal{F}_1}, \| \mu_{rem,n}\| \geq \rho }  P[\phi_{\alpha}=1] \geq 1- \delta \}  \\
  &= \inf \{\rho > 0,  \sup_{ {\mathcal{F}_1}, \| \mu_{rem,n}\|  \geq \rho }  P[\phi_{\alpha}=0] \leq \delta \} 
\end{align*} 
where $\phi_{\alpha}$ is the test result that corresponds to the test statistics
 \begin{align*}
 T_{\mu,n} &=  \frac{ \|W_{G_{n}}\|^2 }{\Big( \|W_{G^{l}_{n}}\|^2 + \|W_{G^{r}_{n}}\|^2 \Big)} - 2 - \frac{N_n}{D_n}F^{-1}_{N_n, D_n} (\alpha_{\mu,n}) \\
  &=  \frac{\|W_{rem,n}\|^2 }{\Big( \|W_{G^{l}_{n}}\|^2 + \|W_{G^{r}_{n}}\|^2 \Big)}  - \frac{N_n}{D_n}F^{-1}_{N_n, D_n} (\alpha_{\mu,n}) 
 \end{align*}

Let's introduce a test statistic $\hat{T}_{\mu, n} $ for window size $n$
\[\hat{T}_{\mu, n} = \| W_{rem,n}\|^2 - \sigma^2 \chi^2_{N_n}(\alpha_{\mu,n})\] 
and denote its corresponding test as $\hat{\phi}_{\alpha}$.
The following lemma gives an analogous argument of test which will enable us to derive the lower bound of minimal radius in a concise way. We assume that the spanning distance of subgraphs $G^l_n$ and $G^r_n$ is greater than 1, and variance remain unchanged, then we can derive the following properties.

\begin{lemma} \label{lemma:simpleTn}
The minimal radius of $ \| \mu_{rem,n}\| $ derived from test $\phi_{\alpha}$ is the same as from test $\hat{\phi}_{\alpha}$ 
\[\rho_n(\mathcal{F}_1,\phi_{\alpha},\delta) = \rho_n(\mathcal{F}_1,\hat{\phi}_{\alpha},\delta) \]
 \end{lemma}

\begin{proof}
Let $U, V$  be independent random variables,   $U, \geq 0,  V\geq 0$ and  bounded. Let $Z=U V$ be the product of the two random variables.
Define $\mathfrak{z}_\alpha, \mathfrak{v}_{\alpha}$ as the $1-\alpha$ quantile of random variable $Z,$ and $V$. 
Then we have the following relation between the cumulative distribution functions,
\begin{align*}
F_Z(Z \leq U  \mathfrak{v}_{\alpha}) &= P(Z \leq  U  \mathfrak{v}_{\alpha}) \\
&= P(UV\leq  U  \mathfrak{v}_{\alpha},U\geq 0) +P(UV\leq U  \mathfrak{v}_{\alpha},U\leq 0) \\
&= \int_{0}^{\infty} f_U(u)\int_{-\infty}^{\mathfrak{v}_{\alpha}} f_V(v) dvdu \\
&= \int_{0}^{\infty} f_U(u)  F_V(V\leq \mathfrak{v}_{\alpha}) du \\
&= F_V(V\leq \mathfrak{v}_{\alpha}) \\
&= \alpha 
\end{align*}
Thus $F_Z(Z \leq  \mathfrak{z}_{\alpha})=F_Z(Z \leq U  \mathfrak{v}_{\alpha})=F_V(V\leq \mathfrak{v}_{\alpha})$. \\
Analogously, let $Z= \tilde{T}_{\alpha,n}$ and $V=\hat{T}_{\alpha,n}$.
Then $\rho_n(\mathcal{F}_1,\phi_{\alpha},\delta) = \rho_n(\mathcal{F}_1,\hat{\phi}_{\alpha},\delta)$ is satisfied.
\end{proof}
\begin{prop}\label{prop:Baraud}
 Let

\[ \rho^2_{N_n} = \sqrt{2 \log (1+4(1-\alpha_{\mu,n}-\beta)^2){N_n}} \sigma^2 \]
Then, for all $\rho \leq \rho_{N_n}$
\[ \beta(\{\{Y_i, i> t\} \sim {\mathcal{F}_1} , \| \mu_{rem,n}\| =\rho\}) \geq \delta \]
\end{prop}
According to [Baraud 2002], whatever the level-$\alpha$ test $\hat{\phi}_{\alpha_n}$, there exist some observation $\{Y_i, i> t\}$ satisfying $\| \mu_{rem,n}\| =\rho_{D_n}\ $ for which the error of second kind $P[\hat{\phi}_{\alpha}=0] $ is at least $\delta$. This implies the lower bound 

\[ \rho_n(\mathcal{F}_1,\hat{\phi}_{\alpha},\delta) \geq \rho_{D_n} \].

\begin{proof}
The idea of the proof is based on [Baraud2002].
Let $Y_i, i=1,...n \sim \mathcal{F}_0=\mathcal{N}(0, I_d),  Y_i, i =n+1, ...2n \sim \mathcal{F}_1$.

Let $\mu_{\rho}$ be some joint probability measure on 
\[ \mathcal{F}_1[\rho] = \{ \| \mu_{rem,n}\| =\rho \} \]

Setting  $P_{\mu_{\rho}} = \int P d\mu_{\rho}$ and denoting by $\Hat{\Phi}_{\alpha}$ the set of level-$\alpha$ tests, we have

\begin{align*}
\beta(\mathcal{F}_1[\rho]) &=\inf_{\hat{\phi}_{\alpha} \in \Hat{\Phi}_{\alpha}} \sup_{ \mathcal{F}_1[\rho] } P[\phi_{\alpha}=0] \\
& \geq \inf_{\hat{\phi}_{\alpha} \in \Hat{\Phi}_{\alpha}} P_{\mu_{\rho}} [\phi_\alpha =0] \\ 
 &\geq 1-\alpha - \sup_{A|P_0(A) \leq \alpha} | P_{\mu_{\rho}} (A) -P_0(A)| \\
 &\geq 1-\alpha - \sup_{A\in \mathcal{A}} | P_{\mu_{\rho}} (A) -P_0(A)| \\
 &= 1- \alpha - \frac{1}{2}\| P_{\mu_{\rho}} - P_0 \|
\end{align*}
,where  $\|P_{\mu_{\rho}} - P_0 \|$ denotes the total variation nor between the probabilities $P_{\mu_{\rho}}$ and $P_0$.  Assume $P_{\mu_{\rho}}$ is absolutely continuous with respect to  $P_0$.  We denote 
\[ L_{\mu_{\rho}}(y)=\frac{dP_{\mu_{\rho}}}{dP_0}  \]
then
\begin{align*}
 \|P_{\mu_{\rho}} -P_0 \| &= \int |L_{\mu_{\rho}}(y)-1|dP_0(y) \\
 &=E_0[L_{\mu_{\rho}}(y)-1] \\
 & \leq (E_0[L^2_{\mu_{\rho}}(y)]-1)^{1/2} 
\end{align*}
We obtain 
\begin{align*} \label{eq:beta}
 \beta(\mathcal{F}_1[\rho]) &\geq 1- \alpha-\frac{1}{2}(E_0[L^2_{\mu_{\rho}}(y)]-1)^{1/2}\\
 &\geq 1- \alpha-(E_0[L^2_{\mu_{\rho}}(y)]-1)^{1/2}\\
 &\geq 1 - \alpha-\eta 
\end{align*}
where we set $\eta  \geq (E_0[L^2_{\mu_{\rho}}(y)]-1)^{1/2}$, equivalently,  $E_0[L^2_{\mu_{\rho}}(y)] \geq 1+ {\eta}^2.$

Next step is to find some $\rho^*(\eta)$ such that for all $\rho \leq \rho^*(\eta)$,
\begin{equation}\label{eq:Erho}
E_0[L^2_{\mu_{\rho}}(y)] \geq 1+ {\eta}^2,
\end{equation}

so that 
\[ \beta(\mathcal{F}_1[\rho]) \geq 1-\alpha -\eta =\beta \]
is satisfied.

Let $\epsilon=(\epsilon_j)_{j\in I}, I=\{1,...,N_n\}$ be a sequence of Rademacher random variables, i.e. for each $m$, the $\epsilon_j$ are independent and identically distributed random variables taking values  form $\{-1,1\}$ with probability $\frac{1}{2}$.  Let $\rho$ be given and $\mu_{\rho}$ be the distribution of the random variable $\sum_{j\in I} \lambda \epsilon_j e_j$, where $\lambda=\rho/\sqrt{N_n}.$ Clearly $\mu_{\rho}$ supports $\mathcal{F}_1[\rho]$.  We derive $L_{\mu_{\rho}}$ as
\begin{align*}
L_{\mu_{\rho}}(y)&=\frac{dP_{\mu_{\rho}}}{dP_0} \\
&= E_\epsilon \Bigg[ \frac{\exp(-\frac{1}{2}\sum_{j\in I}(y_i-\lambda \epsilon_j)^2}{\exp(-\frac{1}{2}\sum_{j\in I}y^2_i)} \Bigg] \\
&=E_\epsilon \Bigg[ \exp(-\frac{1}{2}\rho^2+\lambda\sum_{j\in I} \epsilon_j y_j) \Bigg] \\
&= e^{-\rho^{2} /2}  \prod_{j\in I} \cosh(\lambda y_j)
\end{align*}
where $yi \sim \mathbf{N}(0,1)$.
Next, we compute $E_0[L^2_{\mu_{\rho}(y)}]$
\begin{align*}
E_0[L^2_{\mu_{\rho}(y)}] &= e^{-\rho^{2} /2}  E_0 \Bigg[ \prod_{j\in I} \cosh^2(\lambda y_j) \Bigg]\\
&= {\cosh(\lambda^2)}^{N_n} \\
& \leq (\exp(\frac{\lambda^4}{2}))^{N_n}\\
&= \exp (\frac{\rho^4}{2N_n})
\end{align*}
By Equation \ref{eq:Erho},  we set 
\begin{align*}
\ln E_0[L^2_{\mu_{\rho}(y)}]  \leq \frac{\rho^4}{2N_n} = \ln(1+ \eta^2)
\end{align*}

Therefore, for $\rho \leq \rho_{N_n} = \sqrt{2N_n \ln (1+ \eta^2)}$,  $\eta = 1- \alpha-\beta$, 
we ensure that 
\[ \beta(\mathcal{F}_1[\rho]) \geq 1-\alpha -\eta =\beta \]

\end{proof}

\textbf{Proof of Proposition \ref{Prop:lower_bound}}

\begin{proof}
The proof of the result for the test statistics $T_n$ is based on analogous arguments assuming $\sigma^2=1$.   Since the distribution of the numerator and de-numerator are independent and $\chi^2$ distribution is a non-negative distribution, by Lemma \ref{lemma:simpleTn},  it is equivalent to consider the following distribution: 
For fix window size $n \in \mathfrak{N}$, we consider the test statistic
\[\hat{T}_{\mu, n} = \| W_{rem,n}\|^2 - \chi^2_{N_n}(\alpha_{\mu,n}) \]

By Proposition \ref{prop:Baraud}, for all $Y_i \in R^d$ such that

\begin{equation}\label{eq:Tn}
\| \mu^{rem}_{G_n}\|^2 \leq \theta(\alpha_{\mu,n}, \beta) \sqrt{nd}
\end{equation}

we have $\mathbb{P}(T_n \leq 0) \geq \beta.$

\end{proof}

\begin{corollary}[Multi-window minimum radius]
Let $\beta \in (0, 1-\alpha_{\mu,n})$ and fix some window size $n \in \mathfrak{N}$
\[\theta(\alpha_{\mu,n}, \beta) = \sqrt{2 \log (1+4(1-\alpha_{\mu,n}-\beta)^2)} \]
If 
\[\sup_{n \in \mathfrak{n}} \{ \| \mu^{rem}_{G_n}\|^2  - \theta(\alpha_{\mu,n}, \beta) \sqrt{nd} \sigma^2 \} \leq 0 \]
then $\Pm(\sup_{n \in \mathfrak{n}} T_{\mu,n}(t) \geq 0 ) \leq 1- \beta $
\end{corollary}

\section{Additional numerical studies}
\begin{figure}[H]
(a) \hspace{100pt}  (b)  \hspace{80pt}
  \centering
  \includegraphics[width=.85\textwidth]{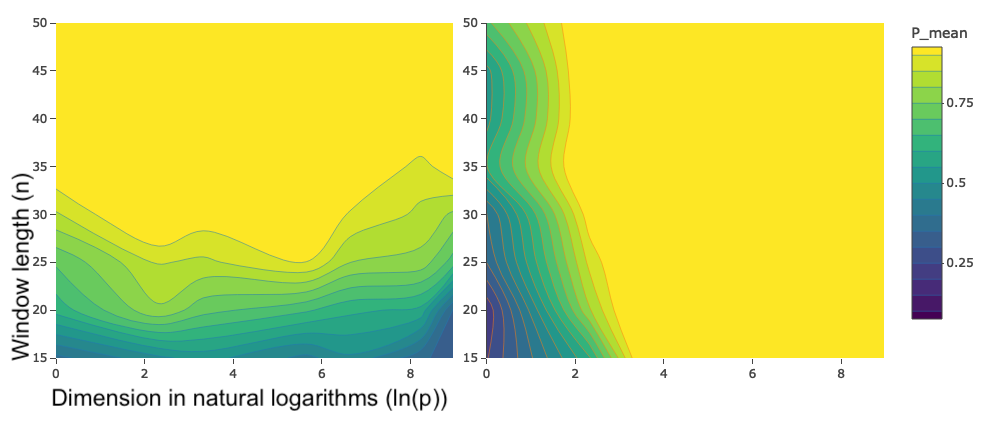}\setlength{\belowcaptionskip}{-5pt}
  \caption{Online detection power  P\_mean: (a) mean change of $\Delta = 1/ \sqrt[3]{d}$, $d$ is the dimension of the data and (b) variance change of $\Sigma = 2 I_d$. Detection power is plot with respect to dimension (in natural logarithms scale) and window length,  with significance less than 5$\%$.}
 \label{fig:Contour_Online}
\end{figure}
For each dimension and each window length, we examine the power within the sample. In general, as seen in Figure \ref{fig:Contour_Online}, OnlineGSR algorithm with complete graph generally demonstrates better testing power under the same significance level.

\bibliographystyle{unsrtnat}
\bibliography{BiblioICML}

\begin{thebibliography}{21}
\providecommand{\natexlab}[1]{#1}
\providecommand{\url}[1]{\texttt{#1}}
\expandafter\ifx\csname urlstyle\endcsname\relax
  \providecommand{\doi}[1]{doi: #1}\else
  \providecommand{\doi}{doi: \begingroup \urlstyle{rm}\Url}\fi

\bibitem[Spokoiny(2009)]{spokoiny2009multiscale}
Vladimir Spokoiny.
\newblock Multiscale local change point detection with applications to
  value-at-risk.
\newblock \emph{The Annals of Statistics}, 37\penalty0 (3):\penalty0
  1405--1436, 2009.

\bibitem[Chen and Gupta(2011)]{chen2011parametric}
Jie Chen and Arjun~K Gupta.
\newblock \emph{Parametric statistical change point analysis: with applications
  to genetics, medicine, and finance}.
\newblock Springer Science \& Business Media, 2011.

\bibitem[Aminikhanghahi and Cook(2017)]{aminikhanghahi2017survey}
Samaneh Aminikhanghahi and Diane~J Cook.
\newblock A survey of methods for time series change point detection.
\newblock \emph{Knowledge and information systems}, 51\penalty0 (2):\penalty0
  339--367, 2017.

\bibitem[Grundy et~al.(2020)Grundy, Killick, and Mihaylov]{grundy2020high}
Thomas Grundy, Rebecca Killick, and Gueorgui Mihaylov.
\newblock High-dimensional changepoint detection via a geometrically inspired
  mapping.
\newblock \emph{Statistics and Computing}, 30\penalty0 (4):\penalty0
  1155--1166, 2020.

\bibitem[Kuleshov et~al.(2016)Kuleshov, Jiang, Zhou, Jahanbani, Batzoglou, and
  Snyder]{kuleshov2016synthetic}
Volodymyr Kuleshov, Chao Jiang, Wenyu Zhou, Fereshteh Jahanbani, Serafim
  Batzoglou, and Michael Snyder.
\newblock Synthetic long read sequencing reveals the composition and
  intraspecies diversity of the human microbiome.
\newblock \emph{Nature biotechnology}, 34\penalty0 (1):\penalty0 64, 2016.

\bibitem[Baringhaus and Gaigall(2017)]{baringhaus2017hotelling}
Ludwig Baringhaus and Daniel Gaigall.
\newblock Hotelling’s t2 tests in paired and independent survey samples: An
  efficiency comparison.
\newblock \emph{Journal of Multivariate Analysis}, 154:\penalty0 177--198,
  2017.

\bibitem[James et~al.(1992)James, James, and Siegmund]{james1992asymptotic}
Barry James, Kang~Ling James, and David Siegmund.
\newblock Asymptotic approximations for likelihood ratio tests and confidence
  regions for a change-point in the mean of a multivariate normal distribution.
\newblock \emph{Statistica Sinica}, pages 69--90, 1992.

\bibitem[Siegmund et~al.(2011)Siegmund, Yakir, and
  Zhang]{siegmund2011detecting}
David Siegmund, Benjamin Yakir, and Nancy~R Zhang.
\newblock Detecting simultaneous variant intervals in aligned sequences.
\newblock \emph{The Annals of Applied Statistics}, pages 645--668, 2011.

\bibitem[Harchaoui et~al.(2009)Harchaoui, Moulines, and
  Bach]{harchaoui2009kernel}
Zaid Harchaoui, Eric Moulines, and Francis~R Bach.
\newblock Kernel change-point analysis.
\newblock In \emph{Advances in neural information processing systems}, pages
  609--616, 2009.

\bibitem[Wang and Samworth(2018)]{wang2018high}
Tengyao Wang and Richard~J Samworth.
\newblock High dimensional change point estimation via sparse projection.
\newblock \emph{Journal of the Royal Statistical Society: Series B (Statistical
  Methodology)}, 80\penalty0 (1):\penalty0 57--83, 2018.

\bibitem[Friedman and Rafsky(1979)]{friedman1979multivariate}
Jerome~H Friedman and Lawrence~C Rafsky.
\newblock Multivariate generalizations of the wald-wolfowitz and smirnov
  two-sample tests.
\newblock \emph{The Annals of Statistics}, pages 697--717, 1979.

\bibitem[Rosenbaum(2005)]{rosenbaum2005exact}
Paul~R Rosenbaum.
\newblock An exact distribution-free test comparing two multivariate
  distributions based on adjacency.
\newblock \emph{Journal of the Royal Statistical Society: Series B (Statistical
  Methodology)}, 67\penalty0 (4):\penalty0 515--530, 2005.

\bibitem[Chen et~al.(2015)Chen, Zhang, et~al.]{chen2015graph}
Hao Chen, Nancy Zhang, et~al.
\newblock Graph-based change-point detection.
\newblock \emph{The Annals of Statistics}, 43\penalty0 (1):\penalty0 139--176,
  2015.

\bibitem[Enikeeva and Harchaoui(2019)]{enikeeva2019high}
Farida Enikeeva and Zaid Harchaoui.
\newblock High-dimensional change-point detection under sparse alternatives.
\newblock \emph{The Annals of Statistics}, 47\penalty0 (4):\penalty0
  2051--2079, 2019.

\bibitem[Liu et~al.(2021)Liu, Gao, and Samworth]{liu2021minimax}
Haoyang Liu, Chao Gao, and Richard~J Samworth.
\newblock Minimax rates in sparse, high-dimensional change point detection.
\newblock \emph{The Annals of Statistics}, 49\penalty0 (2):\penalty0
  1081--1112, 2021.

\bibitem[Kirch(2008)]{kirch2008bootstrapping}
Claudia Kirch.
\newblock Bootstrapping sequential change-point tests.
\newblock \emph{Sequential Analysis}, 27\penalty0 (3):\penalty0 330--349, 2008.

\bibitem[Dong et~al.(2020)Dong, Nakayama, Tuffin, and
  L’Ecuyer]{dong2020tutorial}
Hui Dong, M~Nakayama, B~Tuffin, and P~L’Ecuyer.
\newblock A tutorial on quantile estimation via monte carlo.
\newblock In \emph{Monte Carlo and Quasi-Monte Carlo Methods}, pages 3--30.
  Springer International Publishing, 2020.

\bibitem[Kim et~al.(2006)Kim, Gribbin, Muller, and Taylor]{kim2006analytic}
Hae-Young Kim, Matthew~J Gribbin, Keith~E Muller, and Douglas~J Taylor.
\newblock Analytic, computational, and approximate forms for ratios of
  noncentral and central gaussian quadratic forms.
\newblock \emph{Journal of Computational and Graphical Statistics}, 15\penalty0
  (2):\penalty0 443--459, 2006.

\bibitem[Spokoiny(1996)]{spokoiny1996adaptive}
Vladimir~G Spokoiny.
\newblock Adaptive hypothesis testing using wavelets.
\newblock \emph{The Annals of Statistics}, 24\penalty0 (6):\penalty0
  2477--2498, 1996.

\bibitem[Baraud et~al.(2003)Baraud, Huet, and Laurent]{baraud2003adaptive}
Yannick Baraud, Sylvie Huet, and B{\'e}atrice Laurent.
\newblock Adaptive tests of linear hypotheses by model selection.
\newblock \emph{The Annals of Statistics}, 31\penalty0 (1):\penalty0 225--251,
  2003.

\bibitem[Baraud(2002)]{baraud2002non}
Yannick Baraud.
\newblock Non-asymptotic minimax rates of testing in signal detection.
\newblock \emph{Bernoulli}, pages 577--606, 2002.

\end{thebibliography}

\end{document}